\documentclass{article}

\usepackage{fullpage}
\usepackage{mathtools,amssymb,amsfonts}
\usepackage{amsthm}
\usepackage{bm}
\usepackage{mathrsfs,dsfont}
\usepackage{mleftright}
\usepackage{derivative}
\usepackage[short]{optidef}
\usepackage{algorithm}
\usepackage{algpseudocodex}
\usepackage{graphicx}
\usepackage[export]{adjustbox}
\usepackage{subcaption}
\usepackage{hyperref}
\usepackage{cleveref}
\usepackage{textcomp}
\usepackage{xcolor}
\usepackage[backend=biber,style=ieee,natbib=true]{biblatex}
\usepackage[activate={true,nocompatibility},final,tracking=true,kerning=true,spacing=true,factor=1100,stretch=10,shrink=10]{microtype}

\microtypecontext{spacing=nonfrench}

\theoremstyle{plain}
\newtheorem{theorem}{Theorem}[section]

\newtheorem{lemma}[theorem]{Lemma}
\newtheorem{corollary}[theorem]{Corollary}
\theoremstyle{definition}
\newtheorem{definition}[theorem]{Definition}

\theoremstyle{remark}
\newtheorem{remark}[theorem]{Remark}

\DeclarePairedDelimiter\norm{\lVert}{\rVert}
\DeclarePairedDelimiter\abs{\lvert}{\rvert}
\DeclarePairedDelimiter\inner{\langle}{\rangle}
\DeclarePairedDelimiterX{\KLdiv}[2]{(}{)}{#1\;\delimsize\|\;#2}

\newcommand{\ac}{\text{ac}}
\newcommand{\cC}{\mathcal{C}}
\newcommand{\cS}{\mathcal{S}}
\newcommand{\cP}{\mathcal{P}}
\newcommand{\cT}{\mathcal{T}}

\newcommand{\dsR}{\mathds{R}}

\newcommand{\pushforward}{\sharp}

\DeclareMathOperator{\id}{id}
\DeclareMathOperator{\E}{\mathds{E}}
\DeclareMathOperator{\Cov}{Cov}
\DeclareMathOperator{\Tan}{Tan}
\DeclareMathOperator{\Proj}{\Pi}
\DeclareMathOperator{\flow}{\mathbf{T}}
\DeclareMathOperator{\PS}{\mathcal{P}}
\DeclareMathOperator{\Wass}{\mathcal{W}}
\DeclareMathOperator{\TM}{\mathcal{T}}
\DeclareMathOperator{\PT}{\mathit{P}}
\DeclareMathOperator{\PTH}{\mathcal{P}}
\DeclareMathOperator{\PTM}{\mathscr{P}}
\DeclareMathOperator{\Lip}{Lip}
\DeclareMathOperator{\supp}{supp}
\DeclareMathOperator*{\argmax}{argmax}
\DeclareMathOperator*{\argmin}{argmin}

\title{Another Look at Log-PCA for Probability Measures: A Dynamical Formulation and Statistical Convergence}
\author{
\begin{tabular}{cc}
\begin{minipage}[t]{0.45\textwidth}
\centering
\textbf{Peng Xu}\\
\small \textit{Department of Statistics}\\
\textit{University of Illinois Urbana-Champaign}\\
\texttt{pengxu1@illinois.edu}
\end{minipage}
&
\begin{minipage}[t]{0.45\textwidth}
\centering
\textbf{Changbo Zhu}\\
\small \textit{Department of ACMS}\\
\textit{University of Notre Dame}\\
\texttt{czhu4@nd.edu}
\end{minipage}
\\[4.5em]
\begin{minipage}[t]{0.45\textwidth}
\centering
\textbf{Young-Heon Kim}\\
\small \textit{Department of Mathematics}\\
\textit{University of British Columbia}\\
\texttt{yhkim@math.ubc.ca}
\end{minipage}
&
\begin{minipage}[t]{0.45\textwidth}
\centering
\textbf{Xiaohui Chen}\\
\small \textit{Department of Mathematics}\\
\textit{Thomas Lord Department of Computer Science}\\
\textit{University of Southern California}\\
\texttt{xiaohuic@usc.edu}
\end{minipage}
\end{tabular}
}

\begin{document}

\maketitle
\begin{abstract}
This paper is concerned with learning principal variations of random probability measures on $\dsR^m$ under the Wasserstein geometry. We introduce a new dynamical formulation to interpret the log-PCA, a linearized principal geodesic analysis, as a variational approach. Our differentiable version, termed as the Wasserstein Tangential PCA (WT-PCA), captures the local principal modes of geodesic variations of a (weighted) probability measure on the Wasserstein space via its covariance operator at barycenter. Based on the dynamical perspective and leveraging parallel transport structure of the optimal transport problems, we derive a general statistical convergence rate of the empirical WT-PCA when estimated from data in terms of the 2-Wasserstein distance between the population and empirical barycenter reference measures.
\end{abstract}

\section{Introduction}
Principal component analysis (PCA) is a major statistical analysis and machine learning tool for dimensional reduction and visualization of high-dimensional datasets~\cite{Jolliffe2011}. Classical PCA in the Euclidean space is to find the eigenvectors associated with the top eigenvalues of the covariance matrix. Geometrically, PCA can be interpreted as finding the orthogonal directions that maximize the projected data variance to the \emph{linear} subspace spanned by those directions. Recently, efforts for extending the Euclidean PCA to capture variations for a collection of probability measures have been made~\citep{SeguyCuturi2015_NIPS,Bigot2017_GPCA,Cazelles2018_GPCA-logPCA}. Since the Wasserstein space is an infinite-dimensional curved space, one challenge is to define a proper notion of principal mode of variations in the space of probability measures.

In this paper, we take a \emph{variational} and \emph{dynamical} perspective of the Euclidean PCA that has robust generalization to the Wasserstein geometry. Specifically, given input data points $x_1,\dots,x_n$ in the Euclidean space $\dsR^m$, performing the standard PCA to find the first principal mode of variation $g_t = \overline{x}_n + t v$ passing through the mean $\overline{x}_n = n^{-1} \sum_{i=1}^n x_i$ can be reformulated as minimizing the residuals by projecting each data point in the direction $v$:
\begin{equation}\label{eqn:euclidean_pca_variational}
    \hat{v}_1 = \underset{\norm{v}_2 = 1}{\argmin} \sum_{i=1}^n \min_{t \in \dsR}\norm{x_i - (\overline{x}_n + t v)}_2^2.
\end{equation}
Subsequent components $\hat{v}_2,\cdots,\hat{v}_L$ can be found by iteratively solving~\eqref{eqn:euclidean_pca_variational} with the orthogonality constraints $\langle \hat{v}_j, \hat{v}_k \rangle = \delta_{jk}$. Note that the principal modes of linear variations in $\dsR^m$ are (complete) geodesics with constant velocity vector fields $\hat{v}_1,\dots,\hat{v}_L$. The situation is qualitatively different in the Wasserstein space: given a time-independent vector field $v$ at a reference measure $\nu$, the pushforward curve $t\mapsto g_t(v)_\#\nu = (\id + tv)_\#\nu$ is a Wasserstein geodesic only when $v = \nabla\varphi$ for a convex potential $\varphi$, and even then only up to a critical time $t^\star$, beyond which $\id+ tv$ ceases to be optimal. This obstruction makes principal modes of variation harder to interpret for a collection of probability measures than in the Euclidean case.

To mitigate this challenge, our key observation is that the inner variation in $t \in \dsR$ for computing $\hat{v}_1$ in~\eqref{eqn:euclidean_pca_variational} can be solved analytically in the Euclidean space (or more generally in a Hilbert space), and the variational formulation of PCA in~\eqref{eqn:euclidean_pca_variational} is equivalent to the following tangential formulation:%\footnote{Equivalence of~\eqref{eqn:euclidean_pca_variational} and~\eqref{eqn:euclidean_pca_tangential} is derived in the Appendix.}:
\begin{equation}\label{eqn:euclidean_pca_tangential}
    \hat{v}_1 = \argmax_{\norm{v}_2 = 1}\sum_{i=1}^n \mleft( 
    \odv*{\frac{1}{2} \norm{x_i - (\overline{x}_n + t v)}_2^2}{t}_{t=0} \mright)^2.
\end{equation}
Since the objective function of~\eqref{eqn:euclidean_pca_tangential} equals to $\sum_{i=1}^n \langle x_i - \overline{x}_n, v \rangle^2$, one can resort to spectral decomposition of the covariance matrix of $x_1,\dots,x_n$ to solve~\eqref{eqn:euclidean_pca_tangential}. Thus in view of~\eqref{eqn:euclidean_pca_tangential}, to build a unified \emph{differentiable} version of PCA in a general metric space with a fast algorithm, we need two ingredients.
\begin{enumerate}
    \item An measure-valued ``average'' notion of probability measures $\nu_1,\dots,\nu_n$ as the reference measure, a Wasserstein counterpart of $\bar x_n$.
    \item An metric derivative of differentiable curves in the Wasserstein space.
\end{enumerate}
Regarding 1), as in the previous works, e.g.~\cite{Cazelles2018_GPCA-logPCA},
we consider the Wasserstein \emph{barycenter} defined in \cite{AguehCarlier2011} as our reference measure, which can be efficiently computed with gradient methods~\cite{KimYaoZhuChen2025_barycenter-nonconvex-concave,kim2025sobolevgradientascentoptimal,xu2026unifiedapproachcomputingwasserstein}. Regarding 2), we leverage the dynamical formulation of the OT problem based on the Otto calculus (cf. Chapter 18 in~\cite{ambrosio2021lectures}) to derive a principled metric derivative of (squared) Wasserstein distance, which eventually leads to a computationally efficient spectral algorithm for solving the PCA problem in the space of probability measures with principled statistical convergence rates.

\subsection{Our contribution}

In this paper, we first interpret the log-PCA algorithm (\cite{Cazelles2018_GPCA-logPCA} for probability data and \cite{Fletcher2004_PGA} for Riemannian manifolds) from a dynamical and variational perspective and illustrate its connection to the spectral method without explicitly imposing the tangential constraint. When specialized to a finite collection of input probability distributions, our formulation, termed as the \emph{Wasserstein Tangential PCA} (WT-PCA), is expressed as an unconstrained minimization of the Wasserstein covariance operator and practically equivalent to the log-PCA algorithm~\cite{Cazelles2018_GPCA-logPCA}. The benefit of our first-order dynamical formulation (i.e., Otto calculus) is that, together with tools from the second-order parallel transport of the OT problem, we can derive a general statistical rate of convergence of WT-PCA under the standard sampling model. Specializing to $n$ Gaussian inputs with the Bures-Wasserstein barycenter as our reference measure, we obtain convergence rate $O_p(n^{-1/4})$ for the eigenvalues and eigenfunctions of the aligned Wasserstein covariance operators.

\subsection{Related work}
The work \cite{Bigot2017_GPCA} proposed a geodesic PCA (GPCA) on the space of one-dimensional probability measures. Computation of GPCA was considered in~\cite{SeguyCuturi2015_NIPS,vesseron2026on}. Principal geodesic analysis (PGA) was introduced in~\cite{Fletcher2004_PGA} to analyze data on Riemannian manifolds and later extended to the Wasserstein space in~\cite{Wang:2013uj} to analyze image data. On the other hand, linearized PGA was studied in \cite{Sommer2010_linearized-PGA-comparison}, while \cite{Cazelles2018_GPCA-logPCA} introduced the term log-PCA and compared it with GPCA in the Wasserstein space. \cite{santoro2024statisticalinferencebureswassersteinflows} established statistical properties of log-PCA for time-varying covariance flows in the Bures--Wasserstein space (i.e., centered Gaussians).

\subsection{Organization of the paper} 
The rest of the paper has structure as follows. Section~\ref{sec:prelim_OT} provides some background on the theory of optimal transport. Section~\ref{sec:wt-pca} introduces our notion of the Wasserstein tangential PCA (WT-PCA) and presents its connection to the standard Euclidean PCA as a degenerate special case. Section~\ref{sec:stat_consistency} derives the statistical rate of convergence of the empirical WT-PCA estimated from data. Section~\ref{sec:sim} provides some numerical examples for applying WT-PCA to a collection of Gaussian measures. All proofs are relegated to the Appendix.

\section{Preliminaries on optimal transport}
\label{sec:prelim_OT}

In this section, we review some basics of OT theory.

\paragraph{Wasserstein distance} Let $M$ be a connected subset in $\dsR^m$ and $d(x,y) =\norm{x-y}_2$ be the distance between two points $x, y \in M$. Let $\cP(M)$ be the space (i.e., Wasserstein space) of all probability measures on $M$ equipped with the weak-$\ast$ topology. For $p \geq 1$, let $\cP_{p}(M)$ be the subset of probability measures $\mu \in \cP(M)$ with the finite $p$-th moment $\int_{M} d(x, x_0)^p\odif{\mu}(x) < \infty$ for some $x_0 \in M$. Let $\cP_{\ac}(M)$ (or $\cP_{\ac,p}(M)$) be the subset of absolutely continuous probability measures in $\cP(M)$ (or in $\cP_p(M)$). For any $\mu,\nu \in \cP_2(M)$, the Wasserstein distance $\Wass_2(\mu, \nu)$ is defined by
\begin{equation}
\label{eqn:wasserstein_distance}
\Wass_2^2(\mu, \nu)\coloneqq\inf_{\gamma} \int_{M \times M} d^2(x,y)\odif{\gamma}(x,y),
\end{equation}
where the infimum is taken over all probability measures (i.e., couplings) $\gamma$ on $M \times M$ with marginals $\mu$ and $\nu$. Moreover, an optimal coupling $\gamma^*$ exists for problem~\eqref{eqn:wasserstein_distance}.

\paragraph{Barycenter} Next we define the geometric center of probability measures over the Wasserstein space, an analog of averaging operation in the Euclidean space. \cite{AguehCarlier2011} introduced the Wasserstein barycenter, which is also known as the Frech\'{e}t mean for the metric space $(\cP_2(M), \Wass_2)$.
\begin{definition}[Wasserstein barycenter]
\label{defn:wasserstein_barycenter}
Let $\Omega\in\cP_2\bigl(\cP_2(M)\bigr)$ be a probability measure on $\cP_2(M)$. A Wasserstein \emph{barycenter} $\bar{\nu}$ of $\Omega$ is any minimizer among probability measures $\nu \in \cP_2(M)$ of the functional
\begin{equation}
\label{eqn:barycenter}
\nu \mapsto \int_{\cP_2(M)}\Wass_{2}^{2}(\mu, \nu)\odif{\Omega}(\mu).
\end{equation}
\end{definition}

We denote by $\bar{\nu}$ the Wasserstein barycenter of $\Omega$. \cite{AguehCarlier2011} showed existence and uniqueness (and absolute continuous property) of $\bar{\nu}$ are established for finitely supported measures $\Omega \in \cP_2\bigl(\cP_2(M)\bigr)$. For general measures $\Omega$ (also over Riemannian manifolds), we have the following result \cite{KIM2017640}.

\begin{lemma}[Existence and uniqueness of the Wasserstein barycenter]
\label{lem:existence+uniqueness_barycenter}
Let $M$ be a connected and compact $d$-dimensional Riemannian manifold. If $\Omega\bigl(\cP_{\ac}(M)\bigr) > 0$, then there exists a unique Wasserstein barycenter of $\Omega$. Moreover, the Wasserstein barycenter is also an absolutely continuous measure.
\end{lemma}

\paragraph{Optimal transport map} Given two probability measures $\mu,\nu \in \cP_2(M)$ such that $\mu$ does not give mass to small sets (e.g. $\mu \in P_{ac}(M)$), Brenier's theorem \cite{Brenier_polarization_1991} (cf.~Theorem 2.12 in~\cite{villani2003topics}) states that there is a unique optimal transport map $T\colon M \to M$ pushing forward $\mu$ to $\nu$ (i.e., $T_\pushforward\mu = \nu$) such that
\begin{equation}
\label{eqn:OT_map}
\Wass_2^2(\mu, \nu)=\int_M \norm{x-T(x)}_2^2 \odif{\mu}(x).
\end{equation}
In addition, $T = \nabla \psi$ is given by the gradient of some convex function $\psi\colon M \to \dsR$.

\paragraph{Tangent spaces}
One can define the tangent spaces on $\cP_2(M)$ following the Riemannian interpretation of \cite{doi:10.1081/PDE-100002243} (cf.~\cite[Chapter 8.4]{AmbrosioGigliSavare2008}). Namely, the \emph{tangent space}
to $\cP_2(M)$ at $\mu \in \cP_2(M)$ is
\begin{equation}
\label{eqn:tangent_space}
\Tan_{\mu} \cP_2(M) = \overline{\left\{ \nabla \psi\mid \psi \in C_c^\infty(M) \right\}}^{L_2(\mu)},
\end{equation}
where $C_c^\infty(M)$ is the space of smooth functions with compact support in $M$, and $\overline{v}^{L_2(\mu)}$ denotes the $L_2$ closure of a vector field $v$ w.r.t.~$\mu$. Clearly, we have $\Tan_{\mu} \cP_2(M) \subset L_2(\mu)$. If the base measure $\mu \in \cP_2(M)$ is absolutely continuous with respective to the volume measure of $M$,  then $\Tan_{\mu} \cP_2(M)$ is a Hilbert space (cf. Proposition A.33 in~\cite{LottVillani2009}. Further, we can represent
\[
\Tan_{\mu} \cP_2(M)\coloneqq\log_\mu\bigl(\cP_2(M)\bigr),
\]
where $\log_\mu(\nu)\coloneqq T_{\mu \to \nu} - \id$ with $\nu \in \cP_2(M)$.

\paragraph{Benamou-Brenier formula} 
A curve $\mu_{t}\colon[0,1] \rightarrow \mathcal{P}_2(M)$ is said to be absolutely continuous if the following holds for some $g \in L^1([0,1])$:
\begin{align*}
    \Wass_2(\mu_t, \mu_{s}) \leq \int_s^t g(r)\odif{r} \quad \forall s, t \in [0,1], s \leq t.
\end{align*}
The Wasserstein distance $\Wass_2$ has a dynamical representation known as the Benamou-Brenier formula~\cite{BenamouBrenier2000}:
\begin{equation}\label{eqn:benamou-brenier}
    \Wass_2^2(\mu, \nu) = \min_{\mu_t,v_t}\int_0^1 \int_M\norm{v_t(x)}_2^2\odif{\mu}_t(x) \odif{t},
\end{equation}
where the minimization is taken among all pairs $(\mu_t, v_t)$ with an absolutely continuous curve $\mu_t$ in $\cP_2(M)$ and a time-dependent (velocity) vector field $v_t$ such that they solve the continuity equation (in the weak sense):
\begin{equation}\label{eqn:cty_eqn}
    \frac{\partial \mu_t}{\partial t} + \nabla\cdot(\mu_t v_t) = 0
\end{equation}
subject to the boundary conditions $\mu_0 = \mu$ and $\mu_1 = \nu$. 

Furthermore, $v_t$ is called the tangent (velocity) vector field of the absolute continuous curve $\mu_t$ \citep{maa/1228920869}, if $(\mu_t, v_t)$ satisfies the continuity equation, $\norm{v_t}_{\mu_t} \in L^1([0,1]) $, and
\begin{align*}
    \norm{v_t}_{\mu_t} \leq\abs{\mu_t'}\text{ for almost all } t,
\end{align*}
where \[\abs{\mu_t'}\coloneqq \lim_{h \rightarrow 0}\frac{\Wass_2(\mu_{t+h}, \mu_t)}{\abs{h}}\] represents the metric derivative of $\mu_t$. $(\mu_t)$ is said to be regular if $\int\Lip(v_t)\odif{t} < \infty$. For a regular curve $(\mu_t)$, there exists a unique family of maps \(\flow(s,t,x)\colon\supp(\mu_s)\to\supp(\mu_t)\), called \emph{flow maps}, such that for all \(s,t\in[0,1]\), \[\begin{dcases}
\flow(s,s,x)=x,\\
\odv*{\flow(s,\tau,x)}{\tau}_{\tau=t}=v_t\bigl(\flow(s,t,x)\bigr),\\
\flow\bigl(t,r,\flow(s,t,x)\bigr)=\flow(s,r,x),\\
\flow(s,t,\cdot)_{\pushforward}\mu_s=\mu_t.
\end{dcases}\]

\paragraph{Geodesics in $\PS_2$} Let $\mu, \nu \in \cP_{2,ac} (M)$ such that there is a unique optimal transport map $T = \nabla \psi$ for some convex $\psi$ pushing $\mu$ to $\nu$ (i.e., $T_\pushforward \mu = \nu)$ under the cost $d^2(x,y)/2$. Then the \emph{displacement interpolation} (a.k.a. McCann's interpolation~\cite{MCCANN1997153}) from $\mu_0 = \mu$ to $\mu_1 = \nu$ is given by
\begin{equation}
    \label{eqn:displacement_interpolation}
    \mu_t \coloneqq \bigl( (1-t) \id + t T \bigr)_\pushforward\mu,
\end{equation}
which defines the unique constant-speed geodesic in $ \cP_{2,ac} (M)$ connecting $\mu$ and $\nu$.

\section{Wasserstein tangential PCA}
\label{sec:wt-pca}

Now we turn to our main goal to derive a PCA formulation for probability measures by generalizing the variational formulation of the standard PCA in~\eqref{eqn:euclidean_pca_variational}. Consider a distribution $\Omega \in \cP_2\bigl(\cP_2(M)\bigr)$ on the space $\cP_2(M)$ of probability measures with finite second moment. Suppose that $\Omega$ has a unique Wasserstein barycenter $\bar{\nu}$. Then, the residual minimization formulation for the first principal mode of geodesic variation in $\cP_2(M)$ can be defined as any solution of
\begin{mini}[2]
{\scriptstyle\xi\in T_{\bar{\nu}} \cP_2(M)}{\E_{\nu \sim \Omega} \left[\min_{t \in (-\varepsilon, \varepsilon)}\Wass_{2}^{2}\bigl((\id + \xi_{t})_\pushforward {\bar{\nu}}, \nu\bigr) \right] }{\label{eqn:wt-pca_variational}}{}
\addConstraint{\norm{\xi}_{L^2(\bar{\nu})}^2}{=1.}
\end{mini}
where $\xi_t = t(\xi - \id)$ and $\xi$ is a tangent vector in $\Tan_{\bar{\nu}} \cP_2(M)$. Here, $\varepsilon=\varepsilon(\xi) > 0$, depending on $\xi$, is the maximal radius such that the curve $t \mapsto (\id + \xi_{t})_\pushforward {\bar{\nu}}$ is a geodesic passing through $\bar{\nu}$ for $t \in (-\varepsilon, \varepsilon)$. The dependence of $\varepsilon$ on both $\xi$ and the barycenter $\bar{\nu}$ is a serious drawback: especially because $\bar{\nu}$ itself is implicitly specified through the underlying distribution $\Omega$. Translating properties of $\Omega$ into corresponding properties of $\bar{\nu}$ is an open challenge. Existing results address only rather elementary properties such as absolute continuity; see, e.g.~\cite{AguehCarlier2011, KIM2017640}.

Recall the connection between PCA in variational form~\eqref{eqn:euclidean_pca_variational} and in tangential form~\eqref{eqn:euclidean_pca_tangential} in $\dsR^m$. Our idea is to turn the variational formulation~\eqref{eqn:wt-pca_variational} into the Wasserstein derivative form
\begin{maxi}[2]
{\scriptstyle\xi}{\E_{\nu \sim \Omega} \mleft[\odv*{\frac{1}{2}\Wass_{2}^{2}\bigl((\id+\xi_{t}) {\sharp}{\bar{\nu}}, \nu\bigr)}{t}_{t=0} \mright]^{2}}{\label{eqn:wt-pca_tangential}}{}
\addConstraint{\xi}{ \in \Tan_{\bar{\nu}} \cP_2(M)}
\addConstraint{\norm{\xi}_{L^2(\bar{\nu})}^2}{=1.}
\end{maxi}
In contrast to~\eqref{eqn:wt-pca_variational}, the tangential form~\eqref{eqn:wt-pca_tangential} for finding the principal variation mode has the benefit of avoiding specifying size of the  neighborhood $(-\varepsilon, \varepsilon)$ since the Wasserstein derivative depends only on the \emph{local} variations along the geodesic curve $t \mapsto (\id + \xi_t) \sharp \bar{\nu}$. On the other hand, unlike the Euclidean PCA, one cannot emanate this curve for large enough $|t|$ as a geodesic. Thus our PCA formulation, either in variation form~\eqref{eqn:wt-pca_variational} or tangential form~\eqref{eqn:wt-pca_tangential}, is suitable for capturing the principal mode of variations around the barycenter along some geodesic curve that is interpretable.

To equip~\eqref{eqn:wt-pca_tangential} with a practical algorithm, we need to compute the metric derivative of the squared Wasserstein distance.

\begin{theorem}[Derivative of squared Wasserstein distance {\citep[Cor.~5.25]{santambrogio2015optimal}}]
\label{thm:wasserstein_distance_derivative}
Let $M$ be a connected subset of $\dsR^m$ and $\varepsilon > 0$. Let $(\mu_t)_{t \in [-\varepsilon, \varepsilon)}$ be a weakly continuous curve $[-\varepsilon, \varepsilon) \to \cP(M)$ such that $\mu_t \in \cP_{\ac,2}(M)$ for all $t \in (-\varepsilon, \varepsilon)$ and $\mu_t$ solves the continuity equation~\eqref{eqn:cty_eqn} with $v_t = \xi$. Suppose $\xi$ is a locally Lipschitz vector field such that $\int_{-\varepsilon}^{\varepsilon} \int_M \norm{\xi}_2^2\odif{\mu}_t\odif{t} < \infty.$
Then the curve $t \mapsto \mu_t$ is absolutely continuous and for almost everywhere $t \in (-\varepsilon, \varepsilon)$, we have
\begin{equation}
\label{eqn:wasserstein_distance_derivative}
\odv*{\left(\frac{1}{2}\Wass_2^2(\mu_t, \nu) \right)}{t}= \int_{M}\inner{x - T_t(x), \xi(x)}\odif{\mu_t}(x),
\end{equation}
where $T_t$ is the optimal transport map from $\mu_t$ to $\nu$ under the cost function $d(x,y)^2/2$.
\end{theorem}

Combining Theorem~\ref{thm:wasserstein_distance_derivative} at $t = 0$ and the continuity equation~\eqref{eqn:cty_eqn}, we have
\begin{equation}
\label{eqn:wasserstein_distance_derivative_tangential}
\begin{aligned}
&\odv*{\left(\frac{1}{2}\Wass_2^2\bigl((\id+\xi_t)_\pushforward \mu, \nu\bigr) \right)}{t}_{t=0}= \int_{M} \inner{ x - T_{\mu\to\nu}(x), \xi(x) }\odif{\mu}(x),
\end{aligned}
\end{equation}
where $T_{\mu\to\nu}$ is the optimal transport map from $\rho_0^{(1)} = \mu$ to $\rho_0^{(2)} = \nu$. Thus,~\eqref{eqn:wt-pca_tangential} is equivalent to
\begin{maxi}[2]
{\scriptstyle\xi}{\E_{\nu \sim \Omega}[\inner{\xi, T_{\bar{\nu}\to\nu} - \id}_{\overline \nu}^2]}{\label{eqn:tangential_PCA_1PC}}{}
\addConstraint{\xi}{\in\Tan_{\bar{\nu}} \cP_2(M),}
\addConstraint{\norm{\xi}_{L^2(\bar{\nu})}^2}{=1.}
\end{maxi}
where
\begin{equation}
\langle \xi, \xi' \rangle_\mu \coloneqq  \int_{M} \langle \xi(x), \xi'(x) \rangle\odif{\mu}(x)
\end{equation}
for vector fields $\xi, \xi' \in L_2(\mu)$ and $\inner{\cdot, \cdot}$ denotes the Euclidean inner product in $\dsR^d$. Here, $\norm{\xi}_{L^2(\mu)}^2 = \int_M\norm{\xi}_2^2 \odif{\mu}$ is the induced squared norm by $\inner{\cdot, \cdot}_{\mu}$.
The tangent space constraint $\xi_\ell \in \Tan_{\bar{\nu}}\cP_2(M)$ is not computationally convenient to implement. Thanks to the tangential structure in the objective function in~\eqref{eqn:tangential_PCA_1PC}, such constraint can be automatically relaxed to the entire $L^2(\bar{\nu})$ without changing the solution.

\begin{lemma}[Remove the tangential constraint]
\label{lemma:wt-pca_spectral_notanconstraint}
For a distribution $\Omega \in \cP_2\bigl(\cP_2(M)\bigr)$, solution of~\eqref{eqn:tangential_PCA_1PC} is the same as
\begin{argmaxi}[2]
{\scriptstyle\xi}{\E_{\nu \sim \Omega}[\inner{\xi, T_{\bar{\nu}\to\nu} - \id}_{\overline \nu}^2]}{\label{eqn:tangential_PCA_wasserstein_derivative_version_2}}{}
\addConstraint{\xi}{ \in L^2(\bar{\nu})}
\addConstraint{\norm{\xi}_{L^2(\bar{\nu})}^2}{=1.}
\end{argmaxi}
In particular, if $\xi^*$ is optimal for~\eqref{eqn:tangential_PCA_wasserstein_derivative_version_2}, then $\xi^* \in \Tan_{\bar{\nu}} \cP_2(M)$ and $\xi^*$ is also optimal for~\eqref{eqn:tangential_PCA_1PC}.
\end{lemma}
\begin{proof}[Proof of Lemma~\ref{lemma:wt-pca_spectral_notanconstraint}] We apply the orthogonal projection 
\[
\Proj\colon L^2(\bar \nu) \to \Tan_{\bar \nu} \cP_2 (M).
\] 
Write $\xi = \Pi(\xi) + \xi - \Pi(\xi)$.
Notice that for each $\nu$, \[T_{\bar \nu \to \nu} -\id \in \Tan_{\bar \nu} \cP_2(M).\]
Therefore, $\inner{\xi - \Pi(\xi), T_{\bar{\nu}\to\nu} - \id}_{\overline \nu} =0$ for each $\nu$, implying that 
\[
\inner{\xi, T_{\bar{\nu}\to\nu} - \id}_{\overline \nu} =\inner{\Pi(\xi), T_{\bar{\nu}\to\nu} - \id}_{\overline \nu}
\]
Therefore, we see that the problem 
\eqref{eqn:tangential_PCA_1PC} is exactly the same as 
\eqref{eqn:tangential_PCA_wasserstein_derivative_version_2}
finishing the proof. 
\end{proof}

\subsection{WT-PCA: population version}
\label{subsec:wt-pca_pop}

Now, we are ready to introduce the \emph{Wasserstein tangential PCA} (WT-PCA) defined as
\begin{maxi}[2]
{\scriptstyle\xi_1, \dots, \xi_L}{\sum_{l=1}^{L}  \int_{\cP(M)}\inner{\xi_l, T_{\bar{\nu}\to\nu} - \id}_{\overline \nu}^2  \odif{\Omega}(\nu)}{
\label{eqn:tangential_PCA_wasserstein_derivative_version}}{}
\addConstraint{\langle \xi_{i}, \xi_{j} \rangle_{\bar{\nu}}}{ = \delta_{ij}}.
\end{maxi}

Note that the objective function in~\eqref{eqn:tangential_PCA_wasserstein_derivative_version} is quadratic in $\xi_l, l\in[L]$. This leads to our key definition in performing such a tangential-style PCA in $\cP_2(M)$.

\begin{definition}[Wasserstein covariance]
\label{defn;wasserstein_covariance}
Given a reference measure $\bar{\nu} \in \cP_2(M)$, the {\it Wasserstein covariance} of $\Omega$ at $\bar{\nu}$ is defined as
\begin{equation}
\label{eqn:wasserstein_covariance}
\Cov_\Omega (\xi, \xi') \coloneqq \int_{\cP_2(M)} Q_\nu(\xi, \xi')\odif{\Omega}(\nu),
\end{equation}
where 
\begin{equation}
Q_\nu(\xi, \xi') = \inner{T_{\bar{\nu}\to\nu} - \id, \xi}_{\bar{\nu}} \cdot\inner{T_{\bar{\nu}\to\nu} - \id, \xi'}_{\bar{\nu}}
\end{equation}
is a bilinear form on $L^2(\bar{\nu})$ vector fields. In particular, if we restrict $\xi, \xi'$ to the tangent space $T_{\bar{\nu}} \cP_2(M)$, we refer $\Cov_\Omega(\xi, \xi')$ as the \emph{tangential Wasserstein covariance}.
\end{definition}

We may interpret $\Cov_\Omega (\xi, \xi')$ as the covariance of two displacement vector fields $\xi$ and $\xi'$ for moving masses in $M$ with respective to the distribution $\Omega$ in $\cP_2(M)$. Equivalently, we can rewrite the WT-PCA in~\eqref{eqn:tangential_PCA_wasserstein_derivative_version} as a constrained maximization of a quadratic form
\begin{argmaxi}[2]
{\scriptstyle\xi_1, \dots, \xi_L}{\sum_{\ell=1}^{L} \Cov_\Omega(\xi_\ell, \xi_\ell)}
{\label{eqn:tangential_PCA_wasserstein_cov_version}}
{(\xi^*_1, \dotsc, \xi^*_L) \coloneqq}
\addConstraint{\langle \xi_{i}, \xi_{j} \rangle_{\bar{\nu}}}{ = \delta_{ij}}
\end{argmaxi}

Once the velocity vectors $\xi^*_1, \dotsc, \xi^*_L$ are obtained from~\eqref{eqn:tangential_PCA_wasserstein_cov_version}, we can build the principal modes of geodesic variation as \(g_{\ell,t} = \id + t \xi^\ast_\ell\) for $\ell\in[L]$.

\subsection{WT-PCA: empirical version}

In practice, we do not have access to the Wasserstein distribution $\Omega$ on the space $\cP_2(M)$. Rather, we may draw an i.i.d.~samples $\nu_1,\dots,\nu_n$ from $\Omega$. Define the empirical barycenter
\begin{equation}
\label{eqn:empirical_barycenter}
\bar{\nu}_n \coloneqq \argmin_{\mu \in \cP_2(M)}\frac{1}{n}\sum_{i=1}^n \Wass_{2}^{2}(\mu, \nu_i).
\end{equation}

Then the empirical WT-PCA is defined as the solution $ \hat{\xi}_1, \dots, \hat{\xi}_L$ of the following optimization problem:
\begin{maxi}[2]
{\scriptstyle\xi_1, \dots, \xi_L}
{\sum_{\ell=1}^{L}\frac{1}{n}\sum_{i=1}^n   \langle \xi_\ell, T_{\bar{\nu}_n\to\nu_i} - \id \rangle_{\bar{\nu}_n}^2}
{\label{eqn:empirical_tangential_PCA_wasserstein_derivative_version}}
{}
\addConstraint{\langle \xi_{i}, \xi_{j} \rangle_{\bar{\nu}_n}}{ = \delta_{ij}\quad}{i,j\in[L].}
\end{maxi}
where $T_{\bar{\nu}_n\to\nu_i}$ is the optimal transport from $\bar{\nu}_n$ to $\nu_i.$ Note that the empirical version of the WT-PCA is also known as the log-PCA for a finite collection of probability distributions~\cite{Cazelles2018_GPCA-logPCA}, which leads to a spectral Algorithm~\ref{alg:wt_pca} that can be efficiently implemented.

\begin{definition}[Discretized Wasserstein covariance matrix]
Let \(\nu_1,\dotsc,\nu_n,\mu\in\PS_2(M)\) and \(x_1,\dotsc,x_m\) be a sample from \(\mu\). The Discretized Wasserstein covariance matrix \(\hat{Q}\) of \(\nu_1,\dotsc,\nu_n\) at reference \(\mu\) is the block matrix \[\hat{Q}\coloneqq\begin{bmatrix}
\dfrac{Q_{i,j}}{dm-1}
\end{bmatrix}_{i,j=1}^m,\] where $Q_{i,j}=[(T_{\mu\to\nu_k}(x_{i}) - x_{j})^\top (T_{\mu\to\nu_k}(x_{j}) - x_{j})]_{k=1}^n.$
\end{definition}

\begin{remark}[Relation to Euclidean PCA as a degenerate case]
It is interesting to observe that the WT-PCA is an infinite-dimensional generalization of the Euclidean PCA for a point cloud $x_1,\dots,x_n$ in $\dsR^m$. In this case, $\nu_i = \delta_{x_i}$ and the barycenter $\bar{\nu}_n = \delta_{\overline{x}_n}$ are degenerate Dirac delta measures. Thus the optimal transport map $T_{\bar{\nu}_n \to \nu_i}$ is supported on $\overline{x}_n$ and $T_{\bar{\nu}_n \to \nu_i}(\overline{x}_n) = x_i$. As a result, we have
\begin{align*}
\langle \xi, T_{\bar{\nu}_n\to\nu_i} - \id \rangle_{\bar{\nu}_n} = \langle \xi(\overline{x}_n), T_{\bar{\nu}_n \to \nu_i}(\overline{x}_n) - \overline{x}_n \rangle = \langle \xi(\overline{x}_n), x_i - \overline{x}_n \rangle.
\end{align*}
Then the WT-PCA~\eqref{eqn:empirical_tangential_PCA_wasserstein_derivative_version} reduces to the standard PCA in $\dsR^m$:
\begin{maxi}[2]
{\scriptstyle v_1,\dots,v_L \in \dsR^m}{\sum_{l=1}^{L}  v_l^\top S_n v_l}{}{}
\addConstraint{\langle v_i, v_j \rangle}{ = \delta_{ij}}
\end{maxi}
where $S_n = n^{-1} \sum_{i=1}^n (x_i - \overline{x}_n) (x_i - \overline{x}_n)^\top$ is the sample covariance matrix of $x_1,\dots,x_n$ and the optimal $v_l = \xi_l(\overline{x}_n)$ is the $l$-th eigenvector of $S_n$.
\end{remark}

\begin{algorithm}[th]
   \caption{Wasserstein tangential PCA}
   \label{alg:wt_pca}
\begin{algorithmic}
   \State {\bfseries Input:} probability distributions $\nu_1,\dots,\nu_n$, time variation $t$, Monte Carlo repetition number $m$, number of eigenvalues \(L\).
   \State{Initialize the discretized Wasserstein covariance matrix $\hat{Q}_n$ as a zero $(dm) \times (dm)$ matrix}
   \For{$i\gets1$ {\bfseries to} $n$}
   \State{Compute the OT map $T_{\bar{\nu}_n\to\nu_i}$ from $\bar{\nu}_n$ to $\nu_i$}
   \EndFor
   \State{Draw $m$ samples $x_1,\dots,x_m$ from the barycenter $\bar{\nu}_n$}
   \For{$a\gets1$ {\bfseries to} $m$}
   \State{$s_1 \gets d a - d+1$
   
   $e_1\gets d a$
   
   \(\tau_1\gets [T_{\bar{\nu}_n\to\nu_i}(x_{a}) - x_{a}]_{i=1}^n\)}
   \For{$b\gets a$ {\bfseries to} $m$}
   \State{$s_2\gets d b - d+1$
   
   $e_2\gets d b$
   
   \(\tau_2\gets [T_{\bar{\nu}_n\to\nu_i}(x_{b}) - x_{b}]_{i=1}^n\)
   
   $\hat{Q}_n[s_1:e_1, s_2:e_2]\gets (\tau_1\tau_2^\top)/(dm - 1)$
   
   $\hat{Q}_n[s_2:e_2, s_1:e_1]\gets \hat{Q}_n[s_1:e_1, s_2:e_2]^\top$}
   \EndFor
   \EndFor
   \State{Compute the eigenvectors $\hat{v}_1, \dots, \hat{v}_2$ of $\hat{Q}_n$ associated with the largest $L$ eigenvalues
   
   Reshape $\hat{v}_\ell$ to a $d \times m$ matrix as the velocity vectors $[\xi_\ell(x_j)]_{j=1}^m$
   
   Compute the locations $[g_{\ell,t}(x_j) = x_j + t \xi_\ell(x_j)]_j$ along the $l$-th principal mode of variation after time $t$}
\end{algorithmic}
\end{algorithm}

\section{Statistical Consistency}\label{sec:stat_consistency}
Suppose that $\nu_1,\dots,\nu_n$ is an i.i.d.~sample drawn from a common distribution $\Omega$ on $\cP_2 (M)$, we show that, as $n \to \infty$, the distances between estimated eigenfunctions $\hat{\xi}_1,\dotsc,\hat{\xi}_L$ and the true eigenfunctions $\xi^*_1,\dotsc,\xi^*_L$ vanish at a speed that depends primarily on the empirical and population barycenter reference measures. In our proof, parallel transport is applied to translate the vectors $ \hat{\xi}_i$, $\xi_i^{*}$ into the same space to make them comparable. 

We follow \citep{maa/1228920869, soa} to construct parallel transports on $\mathcal{P}_2(M)$ along geodesics. For any $\mu_0, \mu_1 \in \cP_2(M)$, the geodesic $(\mu_t: 0 \leq t \leq 1)$ between them is 
\begin{align*}
    \mu_t \coloneqq\bigl(tT_{0 \rightarrow1}+(1-t)\id\bigr)\pushforward\mu_0 .
\end{align*}
If $( \mu_t)$ is regular, its unique tangent (velocity) vector field $(v_t)$ and flow maps ($\mathbf{T}(s,t, \cdot)$) are given as
\begin{align*}
v_t & \coloneqq (T_{0 \rightarrow1 } - \id) \circ\bigl(tT_{0 \rightarrow 1 } (1-t)\id\bigr)^{-1},\\
    \flow(s, t, \cdot) & \coloneqq\bigl(tT_{0 \rightarrow 1}+(1-t)\id\bigr) \circ\bigl(sT_{0 \rightarrow1}+(1-s)\id\bigr)^{-1}.
\end{align*}
Using the flow maps, we can define the translation $\TM_{s, t}\colon L^2(\mu_s) \rightarrow L^2(\mu_t)$ as \[\TM_{s, t} (u) = u \circ \flow(t, s, \cdot) \text{ for any } u \in L^2(\mu_s).\] Let $\Proj_{\mu}\colon L^2(\mu)\rightarrow \Tan_{\mu}\cP_2(M)$ be the orthogonal projection onto $\Tan_{\mu}\cP_2(M)$. For any $u_0 \in \Tan_{\mu_0} \mathcal{P}_2(M)$, the parallel transport $\PT_{0, 1}\colon \Tan_{\mu_0} \mathcal{P}_2(M) \to \Tan_{\mu_1} \mathcal{P}_2(M)$ can be defined as
\begin{align*}
    \PT_{0, 1} (u_0) \coloneqq \lim_{S \in \mathcal{S}} \mathscr{P}_{s_{m-1}}^1\circ \mathscr{P}_{s_{m-2}}^{s_{m-1}}\circ\dotsb\circ\mathscr{P}_{0}^{s_1}(u_0),
\end{align*}
where $\PTM_{s_1 }^{s_2}(u_{s_1})\coloneqq\Proj_{\mu_{s_2}} ( \cT_{s_1,s_2} (u_{s_1}) )$ for any $\mu_{s_1} \in \Tan_{\mu_{s_1}} \mathcal{P}_2(M)$ and the limit is taken over all partitions $S = \{ 0= s_0 < s_1 < \dotsb < s_{m-1} < s_m =1 \} \in\cS$. For any $0< s<t<1$, the parallel transport between $\Tan_{\mu_s} \mathcal{P}_2(M)$ and $\Tan_{\mu_t} \mathcal{P}_2(M)$ can be constructed similarly. Existence, uniqueness and group property of parallel transport constructed above can be found in \citep{maa/1228920869}. In addition, the following properties hold.
\begin{lemma}\label{lemma:2}
For any $u_s, v_s \in \Tan_{\mu_s} \mathcal{P}_2(M)$ and $ u_t, v_t \in \Tan_{\mu_t} \PS_2(M) $
\begin{align}
   \PT_{s, t} (u_s + v_s) &=\PT_{s, t} (u_s) +\PT_{s, t}(v_s)   \label{eq:1} \\ 
   \inner{\PT_{s, t} (u_s), v_t}_{\mu_t} &=\inner{u_s, \PT_{t, s} (v_t)}_{\mu_s}. \label{eq:2}
\end{align}
\end{lemma}

Following \cite{chen2023wasserstein}, let $\mathcal{H}_{\mu}$ be the space of Hilbert–Schmidt operators from $\Tan_{\mu}\PS_2(M)$ to $\Tan_{\mu}\PS_2(M)$, we define the parallel transport $\PTH_{\mu_0, \mu_1}$ between $\mathcal{H}_{\mu_0}$ and $\mathcal{H}_{\mu_1}$ as 
    \begin{align*}
        \PTH_{\mu_0, \mu_1} [\mathcal{A}](u_1) \coloneqq\PT_{\mu_0, \mu_1} \circ\mathcal{A}\circ\PT_{\mu_1, \mu_0} (u_1),
    \end{align*}
for any $\mathcal{A} \in\mathcal{H}_{\mu_0}, u_1\in\Tan_{\mu_1}\mathcal{P}_2(M)$.

Note that if the bilinear form $\Cov_{\Omega}$ is bounded, we can write 
\begin{align*}
    \Cov_{\Omega}(\xi, \xi')=\inner*{\cC_{\bar{\nu}}(\xi) , \xi' }_{\bar{\nu}},
\end{align*}
where $ \cC_{\bar{\nu}} \in \mathcal{H}_{\bar{\nu}} $ is the covariance operator defined as $\cC_{\bar{\nu}}\coloneqq\int \log_{\bar{\nu}} \nu \otimes \log_{\bar{\nu}}\nu\odif{\Omega}(\nu)$. The eigendecomposition of $\cC_{\bar{\nu}}$ \citep[Thm~7.2.6]{hsing2015theoretical} is given as 
$\cC_{\bar{\nu}} = \sum_{i=1}^{\infty} \lambda_i \xi_i^{*} \otimes \xi_i^{*},$
where $ \lambda_i =\E [\inner*{\log_{\bar{\nu}} \nu, \xi_i }_{\bar{\nu}}^2 ]$ are the eigenvalues. In practice, $\cC_{\bar{\nu}}$ can be estimated as $\widehat{\cC}_{\bar{\nu}_n} = n^{-1} \sum_{i=1}^n \log_{\bar{\nu}_n} \nu_i \otimes \log_{\bar{\nu}_n} \nu_i.$

Let $\phi_{\nu \rightarrow \mu}$ be the Kantorovich potential such that $ \nabla \phi_{\nu \rightarrow \mu} = T_{\nu \rightarrow \mu} $, the following theorem shows that the distance between $\widehat{C}_{\bar{\nu}_n}$ and  $C_{\bar{\nu}}$ is upper bounded in probability by Wasserstein distance between $\overline{\nu}_n$ and $\nu$.

\begin{theorem}[General convergence rate of the Wasserstein covariance operators] \label{thm:1}
Assume that \( M \subset \mathbb{R}^m \) is convex and bounded. Suppose that, with probability 1, the density of i.i.d.~probability measures \(\nu_1, \dots, \nu_n \sim \Omega\) is bounded below by a constant \(c_1 > 0\) on \(M\). If $\phi_{\bar{\nu} \rightarrow \nu_i}:M \to\dsR$ is $\beta$-smooth for all $i \in [n]$, we have
\begin{align*}
\norm[\big]{\PTH_{\bar{\nu}_n, \bar{\nu}}[\widehat{\cC}_{\bar{\nu}_n}]- \cC_{\bar{\nu}}}_{\mathcal{H}_{\bar{\nu}}} = O_p\bigl(\Wass^{2^{-m}}_2 (\bar{\nu}_n, \bar{\nu} ) \bigr).
\end{align*}

If, in addition, $\phi_{\bar{\nu} \rightarrow \nu_i}:M \rightarrow \mathbb
R$ is $\alpha$-strongly convex for all $i$, we have
\begin{align*}
\norm[\big]{\PTH_{\bar{\nu}_n, \bar{\nu}}[\widehat{\cC}_{\bar{\nu}_n}] - \cC_{\bar{\nu}}}_{\mathcal{H}_{\bar{\nu}}} = O_p\bigl(\sqrt{\Wass_2(\bar{\nu}_n, \bar{\nu})}\bigr).
\end{align*}
\end{theorem}
The key ingredient to show the above theorem is to apply the quantitative stability results for optimal transport maps derived in \citep{berman2021convergence}. In the literature, the consistency of Wasserstein barycenter has been established under similar regulartiy conditions as in Theorem \ref{thm:1}. The theorem below is included for the readers' convenience.
\begin{theorem}[Theorem 8.11 \cite{chewi2024statistical}]\label{thm:3}
Assume that for any $\nu \sim \Omega$, $\phi_{\overline{\nu} \rightarrow \nu}$ is $\alpha$-strongly convex and $\beta$-smooth with $\beta - \alpha \in [0,1)$. Then $\overline{\nu}$ is unique and the empirical Wasserstein barycenter $\overline{\nu}_n$ satisfies 
\begin{align*}
    \E [\Wass_2^2(\overline{\nu}_n, \overline{\nu}) ] \leq \frac{4 \sigma^2}{(1-(\beta - \alpha))^2 n},
\end{align*}
where $\sigma^2 = \int \Wass_2^2(\overline{\nu}, \nu)\odif{\Omega}(\nu)$.
\end{theorem}

Next, we present convergence results for Gaussian measures, i.e., $\nu_i = N(0, \Sigma_i)$ for $i=1,2, \cdots, n$. We write $\overline{\nu}_n = N(0, \overline{\Sigma}_n)$ and $\overline{\nu} = N(0, \overline{\Sigma})$. Given a symmetric matrix $\Sigma$, we use $\sigma_\text{min}(\Sigma)$ and $\sigma_\text{max}(\Sigma)$ to denote its smallest and largest singular value respectively.

\begin{theorem} \label{thm:2}
Let $\nu_i = N(0, \Sigma_i)$ and $\Sigma_i, \Sigma_{\overline{\nu}} \in \mathbb{S}$, where \[\mathbb{S}\coloneqq\{ \Sigma\mid 0< c \leq \sigma_\text{min}(\Sigma) < \sigma_\text{max}(\Sigma)\leq C <\infty \}\] for some constants $0 < c<C<\infty$. If $\Sigma_{\overline{\nu}_n} \in \mathbb{S}$,  then 
\begin{align*}
&\norm{\mathcal{P}_{\overline{\nu}_n, \overline{\nu}}[\widehat{\mathcal{C}}_{\overline{\nu}_n}] - \mathcal{C}_{\overline{\nu}}}_{\mathcal{H}_{\overline{\nu}}}= O(\norm{\Sigma_{\overline{\nu}_n}- \Sigma_{\overline{\nu}} }_2^{1/2} + \norm{\Sigma_{\overline{\nu}_n}^{-1}- \Sigma_{\overline{\nu}}^{-1}}_2^{1/2}).
\end{align*}
\end{theorem}
The convergence of $\Sigma_{\overline{\nu}_n}$ to $ \Sigma_{\overline{\nu}} $ has been established by \citet{kroshnin2021statistical}. In particular, both $\norm{\Sigma_{\overline{\nu}_n}- \Sigma_{\overline{\nu}}}_2$ and $\norm{\Sigma_{\overline{\nu}_n}^{-1}- \Sigma_{\overline{\nu}}^{-1}}_2$ can be shown to be upper bounded by $\norm{\Sigma_{\overline{\nu}}^{-1/2}\Sigma_{\overline{\nu}_n} \Sigma_{\overline{\nu}}^{-1/2} - I}_F $, which is of order $O(1/\sqrt{n})$ with high probability under weaker assumptions \citet[Thm.~2.3]{kroshnin2021statistical}.
\begin{corollary}
Suppose $\Sigma_i$ are i.i.d. random draws from a distribution on $\mathbb{S}$ Under the assumptions of \Cref{thm:2},
\[\norm{\mathcal{P}_{\overline{\nu}_n, \overline{\nu}}[\widehat{\mathcal{C}}_{\overline{\nu}_n}] - \mathcal{C}_{\overline{\nu}}}_{\mathcal{H}_{\overline{\nu}}}=O_p(n^{-1/4}).\]
\end{corollary}

For the consistency of eigenvalues and eigenvectors, the following results are derived directly from Lemmas 4.2 and 4.3 of \citet{bosq2000linear}: 
\begin{corollary}
For any $i \geq 1$, \[\sup_{i}\abs{\hat{\lambda}_i - \lambda_i}\leq\norm[\big]{\mathcal{P}_{\bar{\nu}_n, \bar{\nu}} [\widehat{\cC}_{\bar{\nu}_n}] - \cC_{\bar{\nu}}}_{\mathcal{H}_{\bar{\nu}}},\] and
\begin{align*}
    &\norm[\big]{\PT_{\bar{\nu}_n, \bar{\nu}}(\hat{\xi}_i) - \xi_i^{*} }_{\bar{\nu}}\leq \frac{2\sqrt{2}}{ \min_{1\leq i' \leq i} (\lambda_{i'} - \lambda_{i'+1})}\cdot\norm[\big]{\mathcal{P}_{\bar{\nu}_n, \bar{\nu}}[\widehat{\cC}_{\bar{\nu}_n}] - \cC_{\bar{\nu}}}_{\mathcal{H}_{\bar{\nu}}}.
\end{align*}
\end{corollary}

\section{Numerical experiments}
\label{sec:sim}

\subsection{Simple Gaussian Example}

We first perform a simulation study of the WT-PCA. The simulation setup has three bivariate Gaussians $N(\mu_i,\Sigma_i), i=1,2,3$ with \[\mu_1=(-4,4)^\top,\qquad\mu_2=(3,4)^\top,\qquad\mu_3=(0,-4)^\top,\] and
\[
\Sigma_1 = \begin{pmatrix}
4 & 0 \\
0 & 1 \\
\end{pmatrix},\qquad
\Sigma_2 = \begin{pmatrix}
1 & 0 \\
0 & 4 \\
\end{pmatrix},\qquad
\Sigma_3 = \begin{pmatrix}
1/2 & -1/4 \\
-1/4 & 1 \\
\end{pmatrix}.
\]

\begin{figure}[htbp!]
	\centering
    \includegraphics[width=0.76\columnwidth]{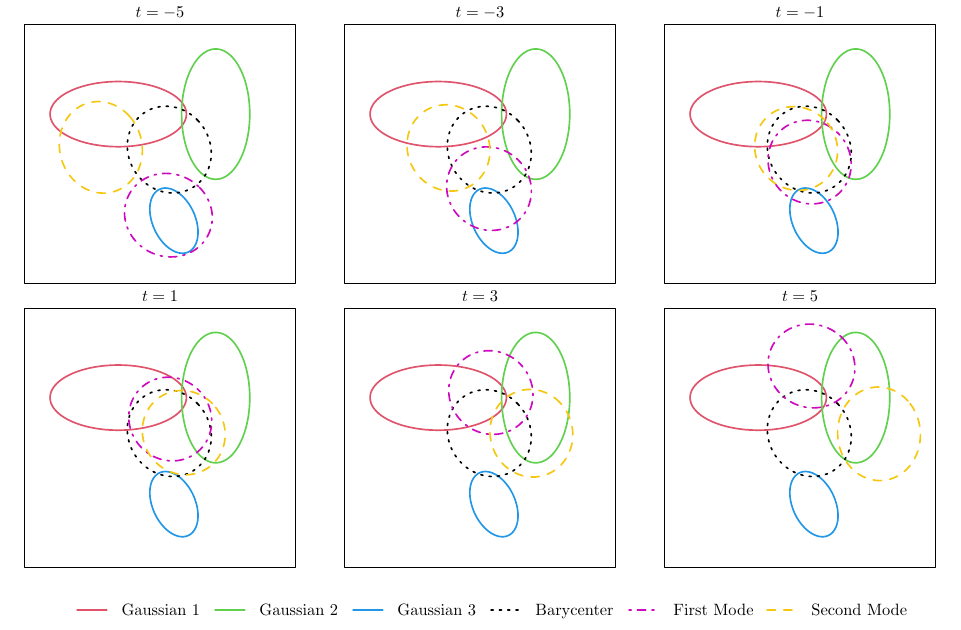}
	\caption{First two principal modes of geodesic variation of the WT-PCA in the example of three bivariate Gaussians, each circle represents the \(95\%\) percentile region of the underlying distribution.}
	\label{fig:three_gaussians_principal_variations}
\end{figure}

In \Cref{fig:three_gaussians_principal_variations}, we plot the first two principal modes of variations at six different time points: $t = -5, -3, -1, 1, 3, 5$. It is interesting to observe that the first principal mode of variation captures the vertical shift of the three Gaussians, and the second one captures the horizontal shift pattern. We also observe that as the time varies, shape and orientation of distributions along the two principal modes of variation evolve continuously to the closest input Gaussian. In addition, we plot the first 10 eigenvalues of the Wasserstein covariance in \Cref{fig:three_gaussians_wasserstein_cov_eigenvalues}. We observe that there are only two positive eigenvalues, corresponding to the vertical and horizontal movement directions in \Cref{fig:three_gaussians_principal_variations}.

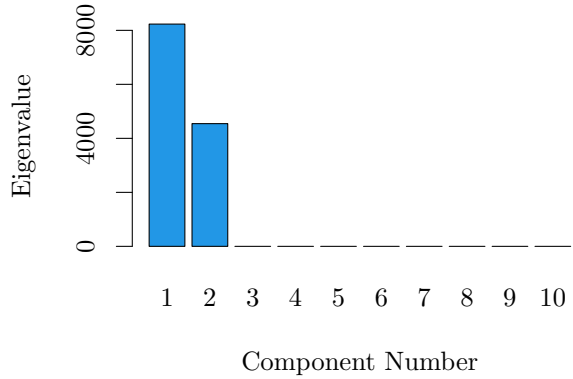
\begin{figure}[ht!]
	\centering
	% Created by tikzDevice version 0.12.3.1 on 2022-12-29 03:16:39
% !TEX encoding = UTF-8 Unicode
\begin{tikzpicture}[x=1pt,y=1pt]
\definecolor{fillColor}{RGB}{255,255,255}
\path[use as bounding box,fill=fillColor,fill opacity=0.00] (0,0) rectangle (231.26,144.54);
\begin{scope}
\path[clip] (  0.00,  0.00) rectangle (231.26,144.54);
\definecolor{drawColor}{RGB}{0,0,0}
\definecolor{fillColor}{RGB}{34,151,230}

\path[draw=drawColor,line width= 0.4pt,line join=round,line cap=round,fill=fillColor] ( 54.34, 48.84) rectangle ( 67.78,132.54);

\path[draw=drawColor,line width= 0.4pt,line join=round,line cap=round,fill=fillColor] ( 70.47, 48.84) rectangle ( 83.91, 95.05);

\path[draw=drawColor,line width= 0.4pt,line join=round,line cap=round,fill=fillColor] ( 86.60, 48.84) rectangle (100.04, 48.84);

\path[draw=drawColor,line width= 0.4pt,line join=round,line cap=round,fill=fillColor] (102.72, 48.84) rectangle (116.16, 48.84);

\path[draw=drawColor,line width= 0.4pt,line join=round,line cap=round,fill=fillColor] (118.85, 48.84) rectangle (132.29, 48.84);

\path[draw=drawColor,line width= 0.4pt,line join=round,line cap=round,fill=fillColor] (134.98, 48.84) rectangle (148.41, 48.84);

\path[draw=drawColor,line width= 0.4pt,line join=round,line cap=round,fill=fillColor] (151.10, 48.84) rectangle (164.54, 48.84);

\path[draw=drawColor,line width= 0.4pt,line join=round,line cap=round,fill=fillColor] (167.23, 48.84) rectangle (180.67, 48.84);

\path[draw=drawColor,line width= 0.4pt,line join=round,line cap=round,fill=fillColor] (183.36, 48.84) rectangle (196.79, 48.84);

\path[draw=drawColor,line width= 0.4pt,line join=round,line cap=round,fill=fillColor] (199.48, 48.84) rectangle (212.92, 48.84);
\end{scope}
\begin{scope}
\path[clip] (  0.00,  0.00) rectangle (231.26,144.54);
\definecolor{drawColor}{RGB}{0,0,0}

\node[text=drawColor,anchor=base,inner sep=0pt, outer sep=0pt, scale=  1.00] at ( 61.06, 26.40) {1};

\node[text=drawColor,anchor=base,inner sep=0pt, outer sep=0pt, scale=  1.00] at ( 77.19, 26.40) {2};

\node[text=drawColor,anchor=base,inner sep=0pt, outer sep=0pt, scale=  1.00] at ( 93.32, 26.40) {3};

\node[text=drawColor,anchor=base,inner sep=0pt, outer sep=0pt, scale=  1.00] at (109.44, 26.40) {4};

\node[text=drawColor,anchor=base,inner sep=0pt, outer sep=0pt, scale=  1.00] at (125.57, 26.40) {5};

\node[text=drawColor,anchor=base,inner sep=0pt, outer sep=0pt, scale=  1.00] at (141.70, 26.40) {6};

\node[text=drawColor,anchor=base,inner sep=0pt, outer sep=0pt, scale=  1.00] at (157.82, 26.40) {7};

\node[text=drawColor,anchor=base,inner sep=0pt, outer sep=0pt, scale=  1.00] at (173.95, 26.40) {8};

\node[text=drawColor,anchor=base,inner sep=0pt, outer sep=0pt, scale=  1.00] at (190.07, 26.40) {9};

\node[text=drawColor,anchor=base,inner sep=0pt, outer sep=0pt, scale=  1.00] at (206.20, 26.40) {10};
\end{scope}
\begin{scope}
\path[clip] (  0.00,  0.00) rectangle (231.26,144.54);
\definecolor{drawColor}{RGB}{0,0,0}

\node[text=drawColor,anchor=base,inner sep=0pt, outer sep=0pt, scale=  1.00] at (133.63,  2.40) {Component Number};

\node[text=drawColor,rotate= 90.00,anchor=base,inner sep=0pt, outer sep=0pt, scale=  1.00] at (  9.60, 90.27) {Eigenvalue};
\end{scope}
\begin{scope}
\path[clip] (  0.00,  0.00) rectangle (231.26,144.54);
\definecolor{drawColor}{RGB}{0,0,0}

\path[draw=drawColor,line width= 0.4pt,line join=round,line cap=round] ( 48.00, 48.84) -- ( 48.00,130.18);

\path[draw=drawColor,line width= 0.4pt,line join=round,line cap=round] ( 48.00, 48.84) -- ( 42.00, 48.84);

\path[draw=drawColor,line width= 0.4pt,line join=round,line cap=round] ( 48.00, 69.17) -- ( 42.00, 69.17);

\path[draw=drawColor,line width= 0.4pt,line join=round,line cap=round] ( 48.00, 89.51) -- ( 42.00, 89.51);

\path[draw=drawColor,line width= 0.4pt,line join=round,line cap=round] ( 48.00,109.85) -- ( 42.00,109.85);

\path[draw=drawColor,line width= 0.4pt,line join=round,line cap=round] ( 48.00,130.18) -- ( 42.00,130.18);

\node[text=drawColor,rotate= 90.00,anchor=base,inner sep=0pt, outer sep=0pt, scale=  1.00] at ( 33.60, 48.84) {0};

\node[text=drawColor,rotate= 90.00,anchor=base,inner sep=0pt, outer sep=0pt, scale=  1.00] at ( 33.60, 89.51) {4000};

\node[text=drawColor,rotate= 90.00,anchor=base,inner sep=0pt, outer sep=0pt, scale=  1.00] at ( 33.60,130.18) {8000};
\end{scope}
\end{tikzpicture}
	\caption{Top 10 eigenvalues of the Wasserstein covariance in the example of three bivariate Gaussians.}
	\label{fig:three_gaussians_wasserstein_cov_eigenvalues}
\end{figure}

\subsection{Grayscale Images}
In this experiment, we consider the MNIST dataset \citep{726791}, which consists of grayscale images of handwritten digits 0--9. All images have resolution \(28 \times 28\) pixels and are treated as histograms on a grid of the same size over \([0,1]^2\). We first computed the Wasserstein barycenter using the FRBary solver \cite{xu2026unifiedapproachcomputingwasserstein} and then applied \Cref{alg:wt_pca} to estimate the principal modes of variation. \Cref{fig:mnist_pca_dg2} displays the first six principal modes for the digit 2. Each row corresponds to a displacement generated by the associated variation vector field for $t\in(-2,2)$. The barycenter was computed from $n=1000$ images, and the Wasserstein covariance matrix was estimated using $1000$ Monte Carlo samples. The first mode primarily captures the slant of the handwritten 2s, while the second mode reflects variation in their height. Additional modes reveal other characteristic deformations. Based on their eigenvalues, these six modes collectively explain approximately $70\%$ of the local variation.

\begin{figure}[!htbp]
\centering
\begin{minipage}[t]{0.45\columnwidth}
	\centering
	\includegraphics[width=0.6\textwidth]{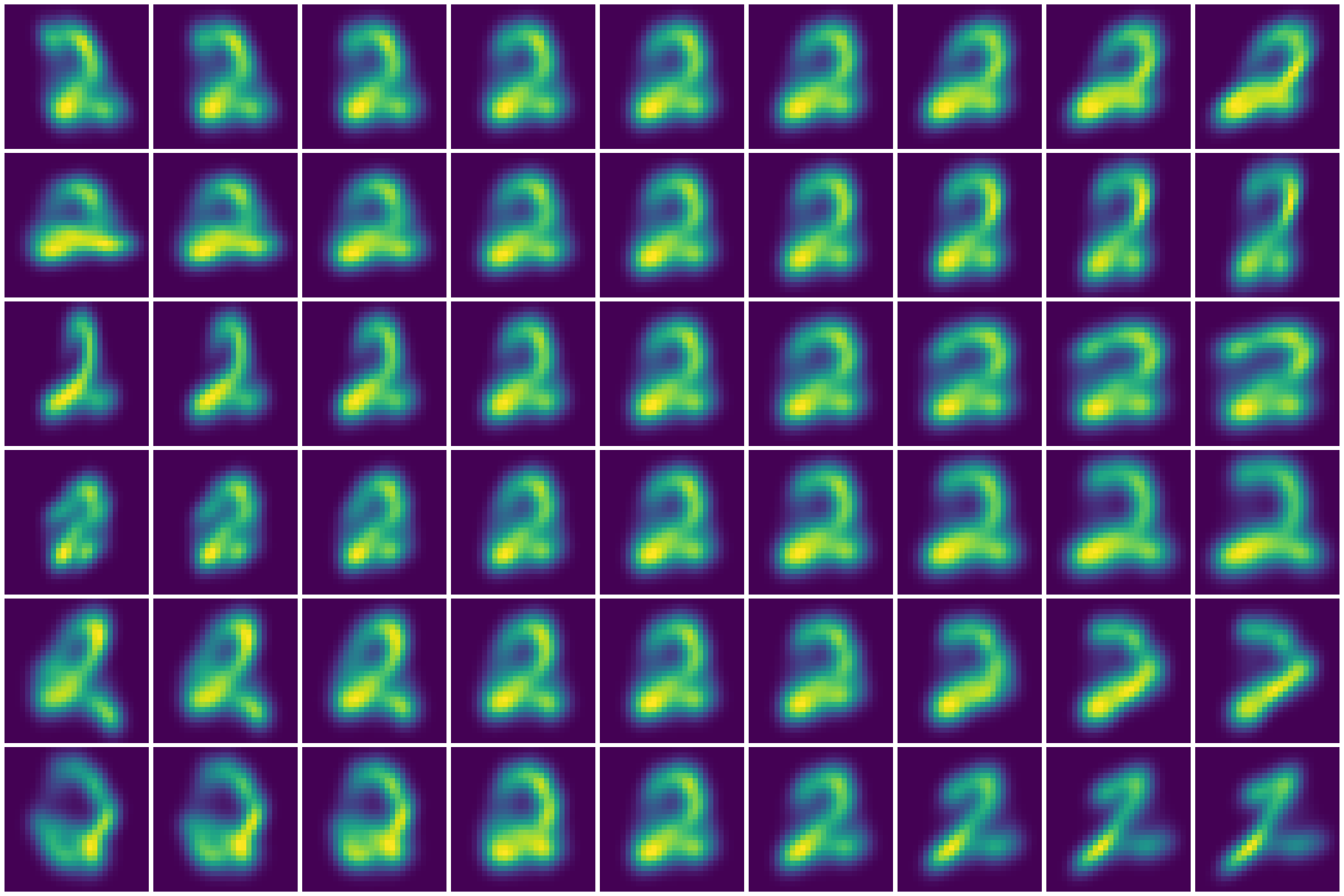}
	\caption{The top 6 principal modes for digit 2. The histogram of the barycenter is shown in the middle column.}
	\label{fig:mnist_pca_dg2}
\end{minipage}
\hfill
\begin{minipage}[t]{0.45\columnwidth}
	\centering
	\includegraphics[width=0.6\textwidth]{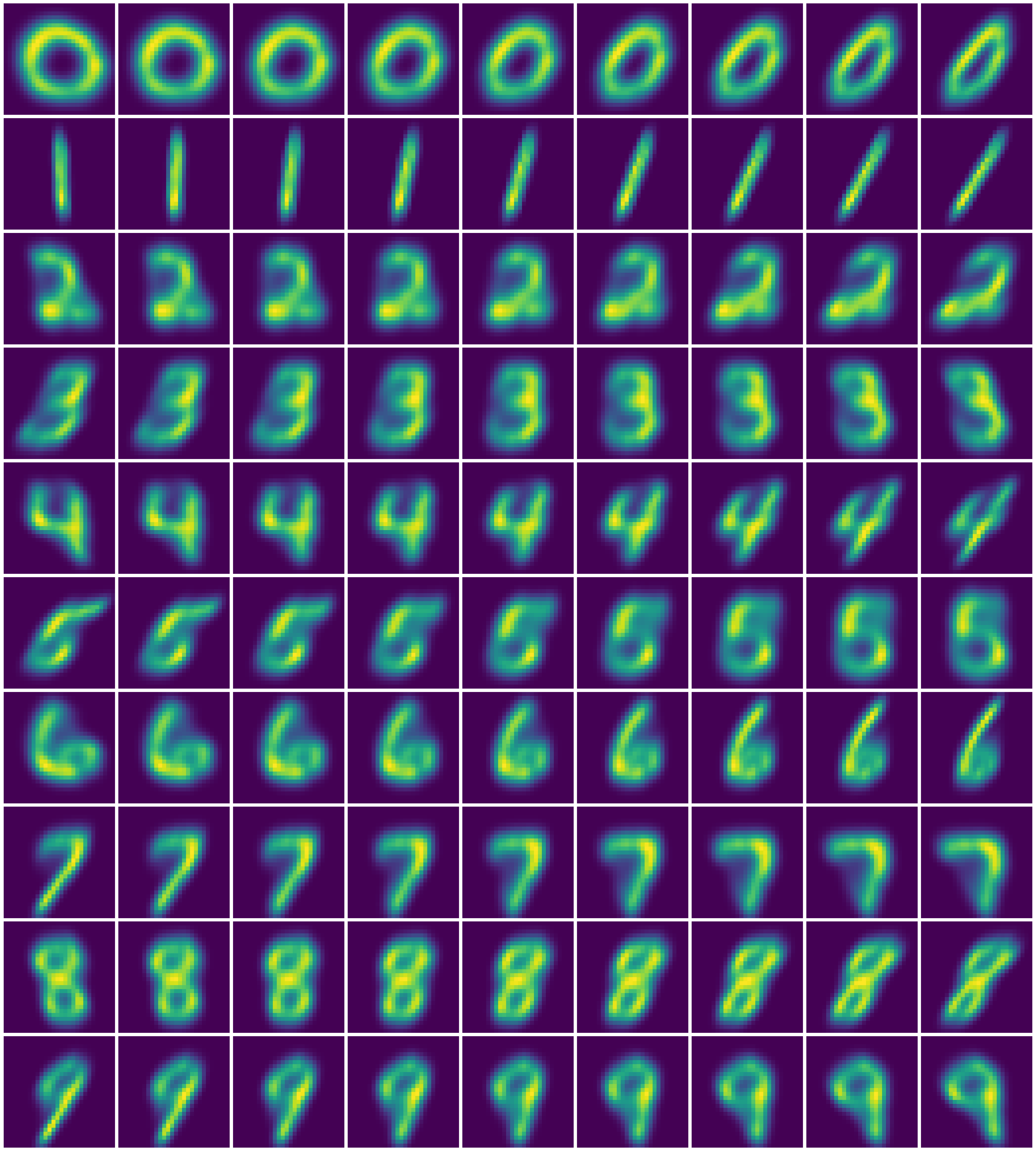}
	\caption{The first principal modes of digits 0--9.}
	\label{fig:mnist_pca_09}
\end{minipage}
\end{figure}

Moreover, we computed the first principal mode for each of the ten digits and visualized the corresponding displacements for \(t \in (-2,2)\) in \Cref{fig:mnist_pca_09}. As before, we used \(n=1000\) images and \(1000\) Monte Carlo samples for each digit. The results indicate that the dominant local variation across all digits is largely attributable to differences in handwriting slant.
\subsection{Color Palette Averaging}
In this experiment, we consider the task of averaging color histograms. The pixels of a three-channel image can be viewed as samples from a color distribution supported on \([0,1]^3\). We downloaded two images from \cite{wallpaper} and computed the Wasserstein barycenter of their color distributions. The image resolutions are \(3160 \times 1846\) and \(2560 \times 1573\), respectively.

The original and recolored images, together with samples from the corresponding color distributions in RGB space are shown in \Cref{fig:mnist_3d_input}. To identify the dominant mode of variation, we sampled \(4096\) pixels from the barycenter and performed WT-PCA. The resulting principal displacement field in RGB space is illustrated in \Cref{fig:mnist_3d_lin}. Remarkably, linear perturbations of the barycenter along the first principal mode produce color distributions that are visually close to the two original images. The eigenvalue spectrum further indicates that the first principal mode alone accounts for approximately \(78\%\) of the local variation.

\begin{figure}[!htbp]
\centering
\begin{subfigure}[t]{0.48\columnwidth}
\centering
\includegraphics[height=0.26\columnwidth,frame]{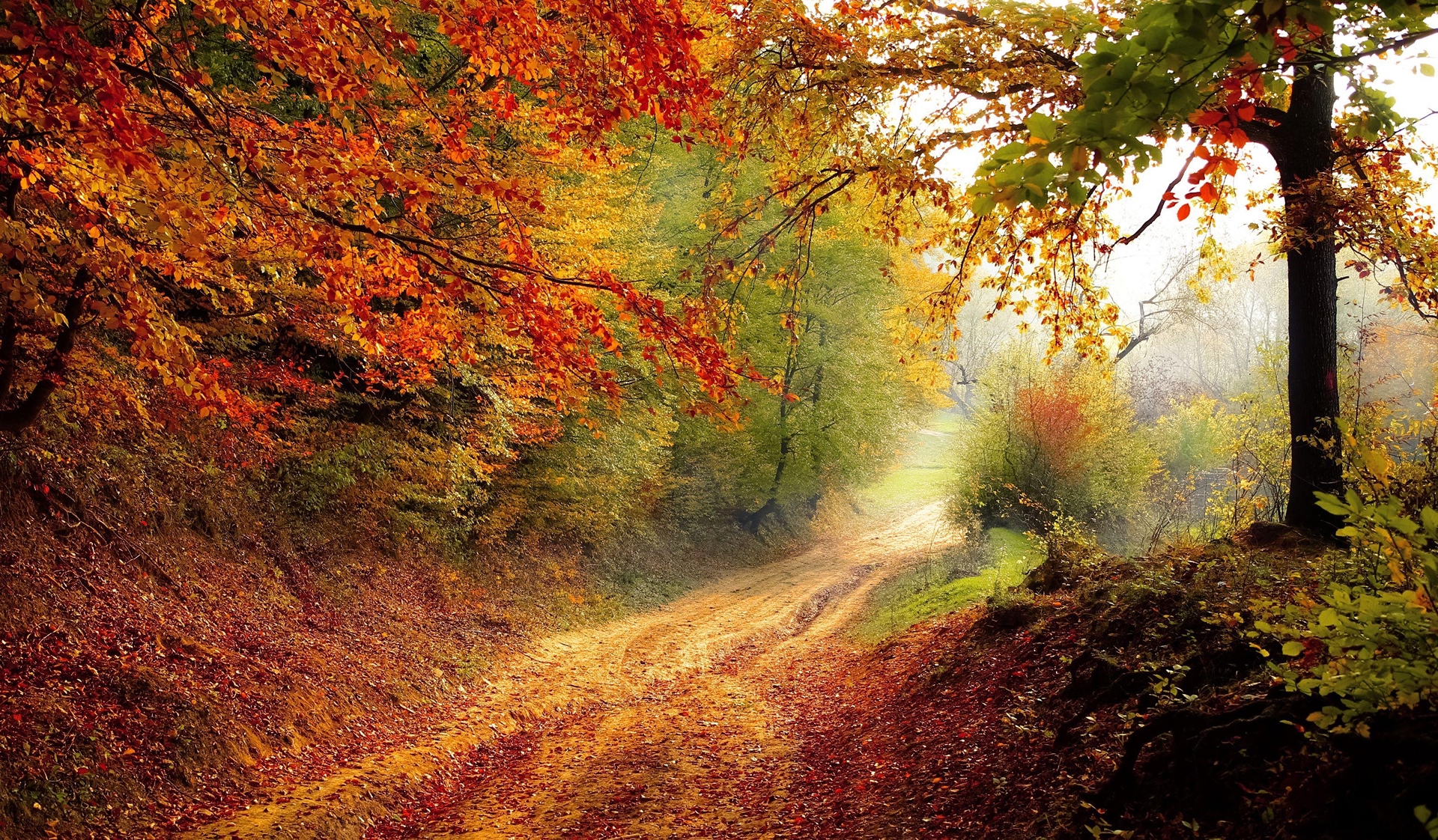}
\quad
\includegraphics[height=0.26\columnwidth,frame]{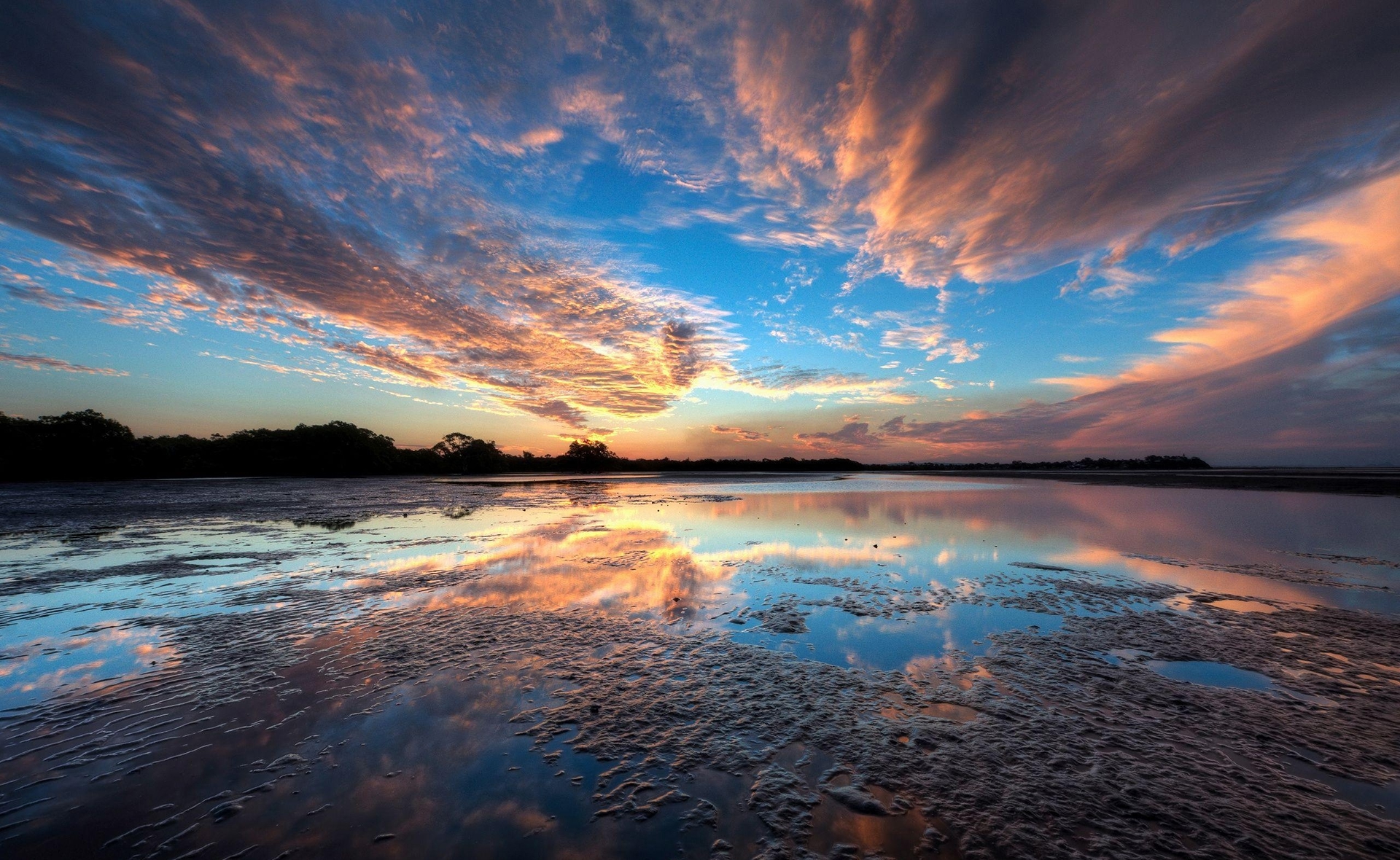}
\caption{Input images.}
\end{subfigure}
~
\begin{subfigure}[t]{0.48\columnwidth}
\centering
\includegraphics[height=0.26\columnwidth,frame]{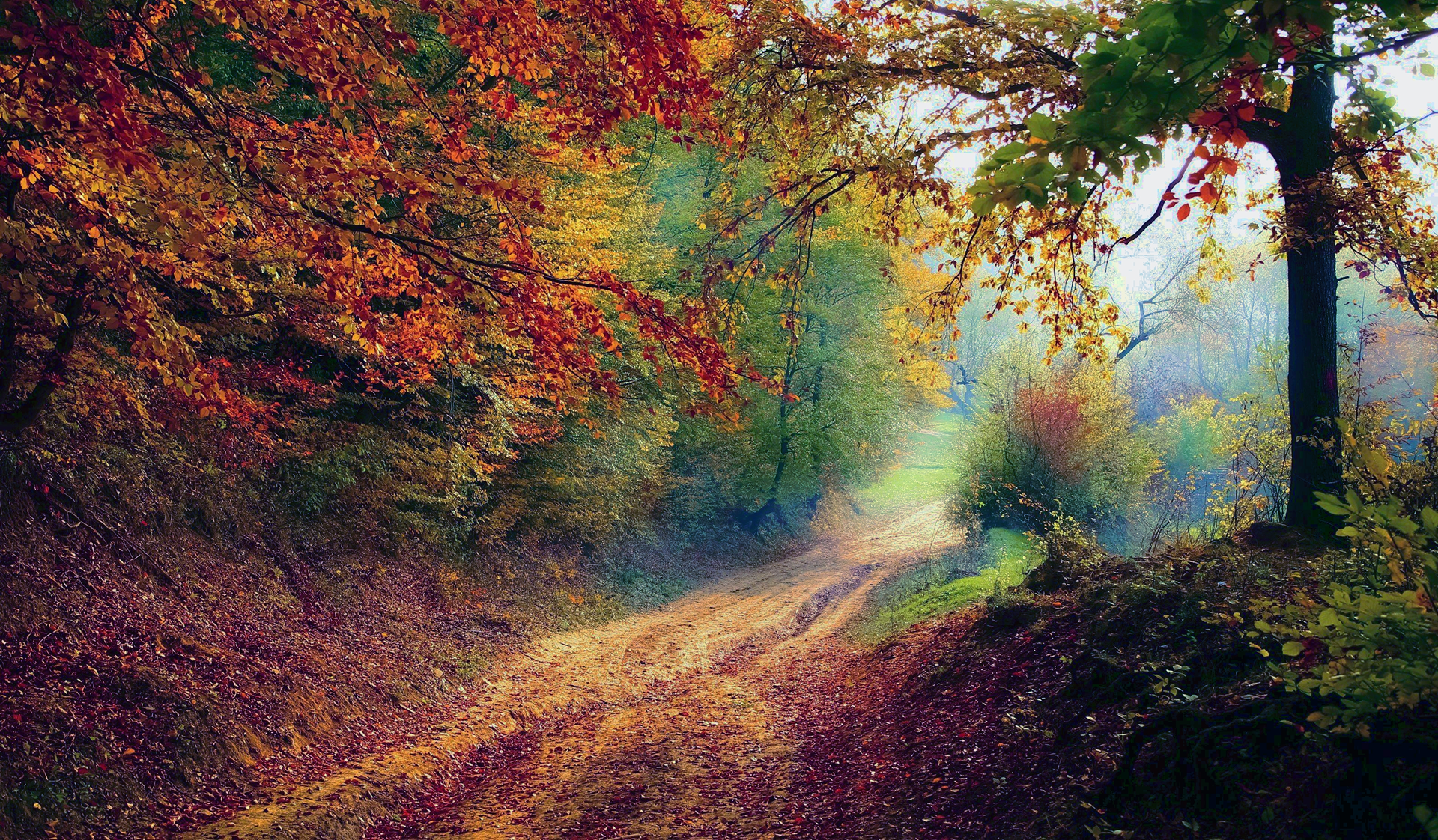}
\quad
\includegraphics[height=0.26\columnwidth,frame]{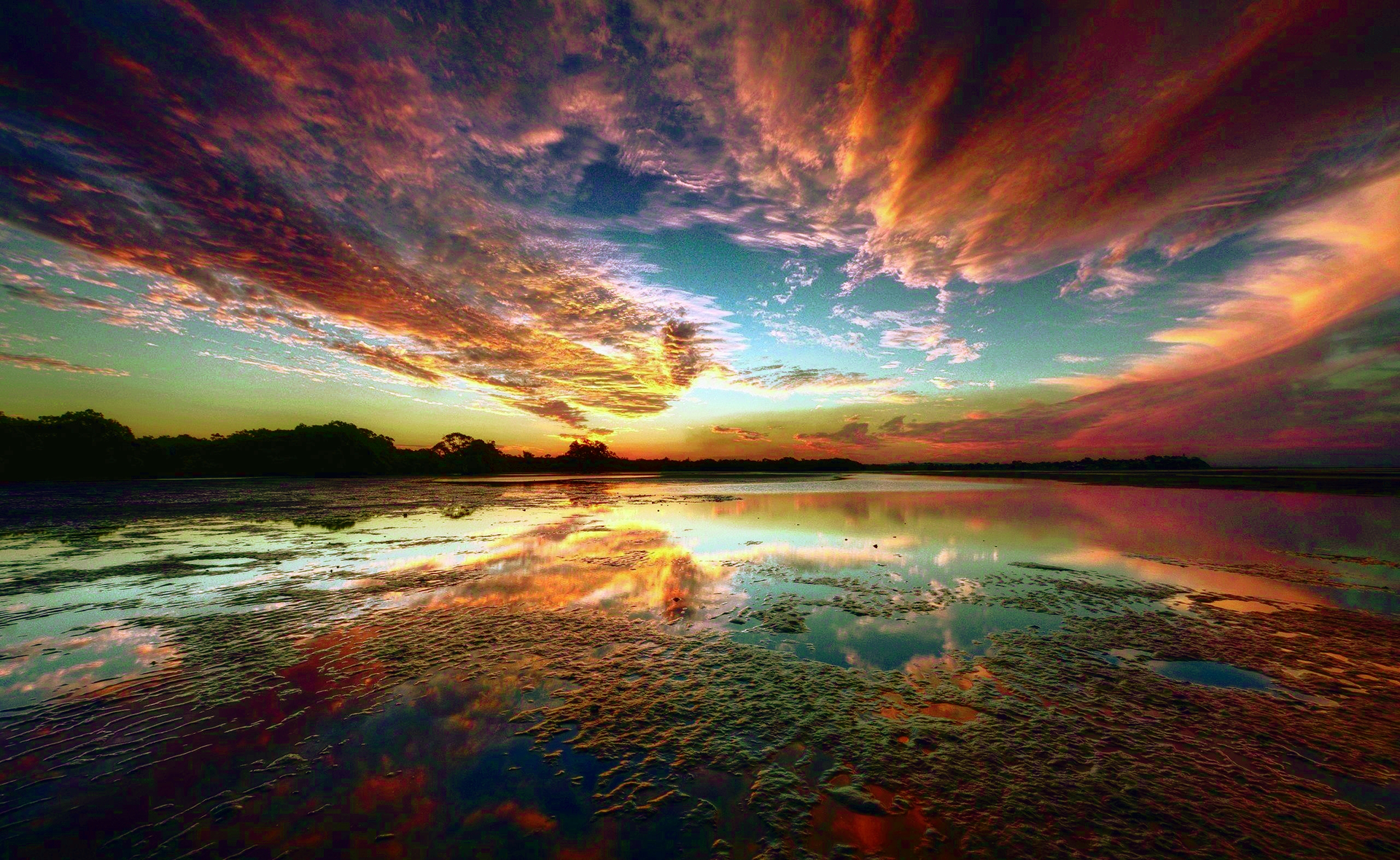}
\caption{Images with averaged color palette.}
\end{subfigure}
\\
\begin{subfigure}[t]{0.5\columnwidth}
\centering
\includegraphics[height=0.45\columnwidth]{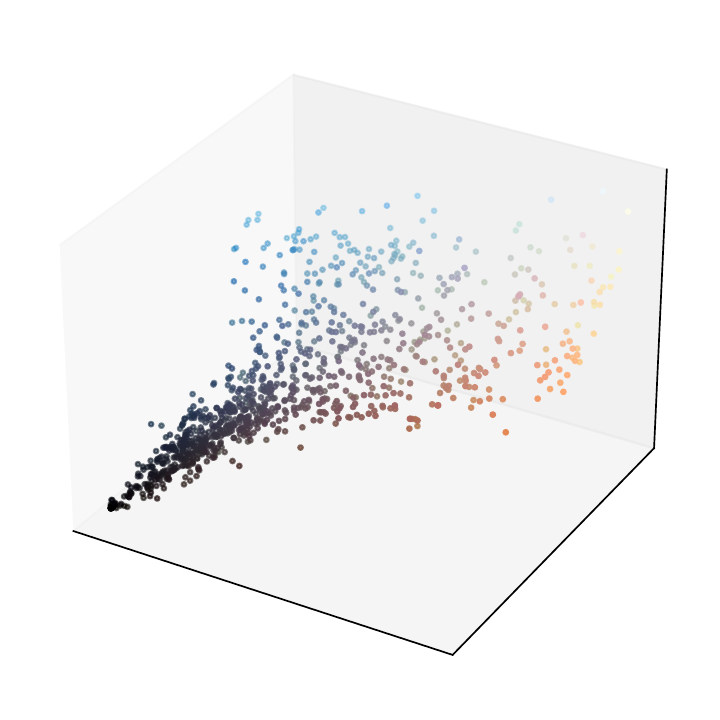}
\quad
\includegraphics[height=0.45\columnwidth]{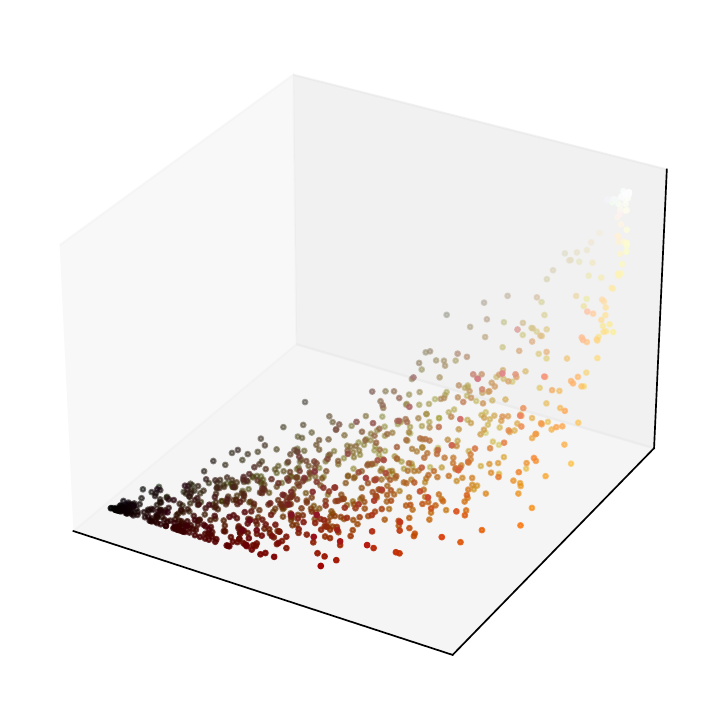}
\caption{Pixels of the original images in the RGB space $[0,1]^3$.}
\end{subfigure}
\caption{Results of color palette averaging.}
\label{fig:mnist_3d_input}
\end{figure}

\begin{figure*}[!htbp]
	\centering
	\includegraphics[width=0.9\textwidth]{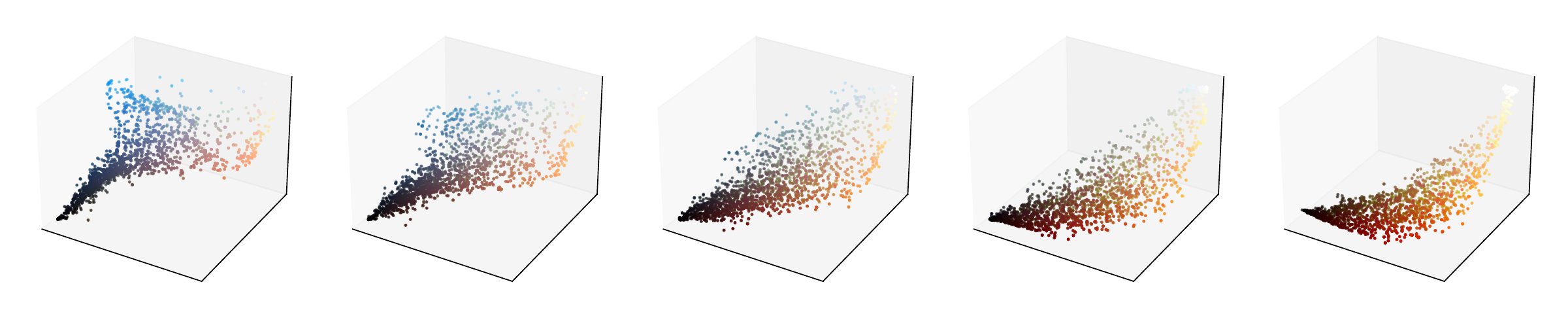}
	\caption{Perturbations of the color-palette barycenter (center) along the first principal mode of variation.}
	\label{fig:mnist_3d_lin}
\end{figure*}

\clearpage
\printbibliography

\clearpage
\appendix
\section{Technical Details}
\subsection{Alternative Proof of Lemma \ref{lemma:wt-pca_spectral_notanconstraint}}
\begin{proof}
First, since the objective function in~\eqref{eqn:tangential_PCA_1PC} is quadratic and thus convex in $\xi$, we shall consider the equivalent constraint $\norm{\xi}_{L^2(\bar{\nu})}^2 \leq 1$. Let $\Pi$ be the projection operator acting on $L^2(\bar{\nu})$ functions to the Hilbert space $\Tan_{\bar{\nu}} \cP_2(M)$. For any vector $\xi \in L^2(\bar{\nu})$, by the variational selection of tangent vectors (cf. Lemma 8.4.2 in~\cite{AmbrosioGigliSavare2008}), there exists a unique $\Pi(\xi) \in \Tan_{\bar{\nu}} \cP_2(M)$ and we have for any vector $v \in \Tan_{\bar{\nu}} \cP_2(M)$,
\[
\int_M \inner{v, \xi - \Pi(\xi)}\odif{\bar{\nu}} = 0.
\]
Let $w = \xi - \Pi(\xi)$ and $\psi$ be a smooth function with compact support in $M$. Since $\nabla \psi \in \Tan_{\bar{\nu}} \cP_2(M)$, integrating by parts yields that
\begin{equation}\label{eqn:divergence_free}
\int_M \psi \; \nabla \cdot (w \bar{\nu}) = 0,
\end{equation}
i.e., $w$ is a divergence-free vector in the sense of duality with respect to smooth test functions $C_c^\infty(M)$. Then by Proposition 8.5.4 in~\cite{AmbrosioGigliSavare2008}, we have
\[
\int_M\inner{T_{\bar{\nu} \to \nu}(x) - x, w(x)}\odif{\bar{\nu}}(x) = 0,
\]
which implies that
\[
\int_M \inner{T_{\bar{\nu} \to \nu} - \id, \xi}\odif{\bar{\nu}}= \int_M \inner{T_{\bar{\nu} \to \nu} - \id, \Pi(\xi)}\odif{\bar{\nu}}.
\]
Thus, the objective function in~\eqref{eqn:tangential_PCA_wasserstein_derivative_version_2} does not change if we replace $\xi$ by $\Pi(\xi)$ as
\begin{equation}\label{eqn:same_objective}
\int_{\cP_2(M)}\inner{T_{\bar{\nu}\to\nu} - \id, \Pi(\xi)}_{\overline \nu}^2  \odif{\Omega}(\nu).
\end{equation}
On the other hand, we can find a sequence $(\psi_k)_{k=1}^\infty$ in $L^2(\bar{\nu})$ such that 
\[
\int_M\inner{\nabla \psi_k, w}\odif{\bar{\nu}} \to \int_M \inner{\Pi(\xi), w}\odif{\bar{\nu}} \quad\text{as}\quad k \to \infty.
\]
Integrating by parts and using~\eqref{eqn:divergence_free}, we see that $\int_M \inner{\Pi(\xi), w }\odif{\bar{\nu}} = 0$. Thus, 
\begin{align*}
    \int_M\norm{\xi}_2^2 \odif{\bar{\nu}} &= \int_M \norm{\Pi(\xi) + w}_2^2\odif{\bar{\nu}}\\
    &= \int_M\norm{\Pi(\xi)}_2^2 \odif{\bar{\nu}} + \int_M \norm{w}_2^2 \odif{\bar{\nu}}\\
    &\geq \int_M\norm{\Pi(\xi)}_2^2 \odif{\bar{\nu}}.
\end{align*}
Then the feasible subset of~\eqref{eqn:tangential_PCA_wasserstein_derivative_version_2} is smaller than that of~\eqref{eqn:wt-pca_tangential}, i.e.,
\begin{equation}\label{eqn:feasibility_inclusion}
\{\norm{\xi}_{L^2(\bar{\nu})}^2 \leq 1 \} \subset \{ \norm{\Pi(\xi)}_{L^2(\bar{\nu})}^2  \leq 1\}.
\end{equation}
Now combining~\eqref{eqn:same_objective} and~\eqref{eqn:feasibility_inclusion} means that for any optimal $\xi^*$ solving~\eqref{eqn:tangential_PCA_wasserstein_derivative_version_2}, we must have $\xi^* = \Pi(\xi^*)$. This proves the equivalence of~\eqref{eqn:wt-pca_tangential} and~\eqref{eqn:tangential_PCA_wasserstein_derivative_version_2}.
\end{proof}

\subsection{Proof of Lemma \ref{lemma:2}}
\begin{proof}
    Equation \eqref{eq:1} follows from the fact that $\PTM_{s}^{t} (u_s+v_s) = \PTM_{s}^{t}(u_s) +\PTM_{s}^{t}(v_s)$. For Equation \eqref{eq:2}, since parallel transport preserves norm \citep[Prop~5.5]{AmbrosioGigliSavare2008} and $\PT_{t,s}\circ\PT_{s, t} = \id$ \citep[Prop~5.13]{AmbrosioGigliSavare2008}, we have 
    \[
        \inner{\PT_{s, t}(u_s), v_t}_{\mu_t}=\inner{\PT_{t, s}\circ\PT_{s, t} (u_s), \PT_{t, s}(v_t)}_{\mu_s}=\inner{u_s, \PT_{t, s}(v_t)}_{\mu_s}.
    \]
\end{proof}

\subsection{Proof of Theorem \ref{thm:1}}
\begin{proof}
Using Lemma \ref{lemma:2}, for any $g \in \Tan_{\bar{\nu}} \mathcal{P}_2(M)$
\begin{align*}
    \PTH_{\bar{\nu}_n, \bar{\nu}}[\widehat{\cC}_{\bar{\nu}_n}](g)&= \PT_{\bar{\nu}_n,\bar{\nu}} \left\{ \frac{1}{n} \sum_{i=1}^n \log_{\bar{\nu}_n} \nu_i \otimes \log_{\bar{\nu}_n} \nu_i \right\}\PT_{\bar{\nu}, \bar{\nu}_n} (g) \\
    &=\PT_{\bar{\nu}_n,\bar{\nu}} \left\{ \frac{1}{n} \sum_{i=1}^n \inner*{\log_{\bar{\nu}_n} \nu_i ,\PT_{\bar{\nu},\bar{\nu}_n} g }_{\bar{\nu}_n}  \log_{\bar{\nu}_n} \nu_i \right\} \\
    &=\frac{1}{n}\sum_{i=1}^n\inner*{\log_{\bar{\nu}_n} \nu_i , \PT_{\bar{\nu},\bar{\nu}_n} g}_{\bar{\nu}_n}\PT_{\bar{\nu}_n,\bar{\nu}}(\log_{\bar{\nu}_n}\nu_i)\\
    &=\frac{1}{n}\sum_{i=1}^n\inner*{\PT_{\bar{\nu}_n,\bar{\nu}}(\log_{\bar{\nu}_n} \nu_i), g}_{\bar{\nu}}\PT_{\bar{\nu}_n,\bar{\nu}} (\log_{\bar{\nu}_n}\nu_i )\\
    &=\left\{\frac{1}{n}\sum_{i=1}^n\PT_{\bar{\nu}_n,\bar{\nu}}(\log_{\bar{\nu}_n} \nu_i)\otimes\PT_{\bar{\nu}_n,\bar{\nu}} ( \log_{\bar{\nu}_n} \nu_i )\right\} g,
\end{align*}
which shows that $\PTH_{\bar{\nu}_n, \bar{\nu}}[\widehat{\cC}_{\bar{\nu}_n}]=n^{-1}\sum_{i=1}^n  P_{\bar{\nu}_n,\bar{\nu}}(\log_{\bar{\nu}_n} \nu_i) \otimes P_{\bar{\nu}_n,\bar{\nu}} ( \log_{\bar{\nu}_n} \nu_i )$. Then,
\begin{align*}
    \PTH_{\bar{\nu}_n, \bar{\nu}}[\widehat{\cC}_{\bar{\nu}_n}] - \cC_{\bar{\nu}}&=\frac{1}{n} \sum_{i=1}^n\PT_{\bar{\nu}_n,\bar{\nu}}(\log_{\bar{\nu}_n} \nu_i) \otimes\PT_{\bar{\nu}_n,\bar{\nu}} ( \log_{\bar{\nu}_n} \nu_i ) - \cC_{\bar{\nu}}\\
    &=\frac{1}{n} \sum_{i=1}^n  \log_{\bar{\nu}} \nu_i \otimes  \log_{\bar{\nu}} \nu_i  - \cC_{\bar{\nu}} \\
    &\quad+\frac{1}{n} \sum_{i=1}^n\bigl(\PT_{\bar{\nu}_n,\bar{\nu}} (\log_{\bar{\nu}_n} \nu_i)  - \log_{\bar{\nu}} \nu_i\bigr) \otimes  \log_{\bar{\nu}} \nu_i  \\
    & \quad+ \frac{1}{n} \sum_{i=1}^n  \log_{\bar{\nu}} \nu_i \otimes (\PT_{\bar{\nu}_n,\bar{\nu}} (\log_{\bar{\nu}_n} \nu_i)  - \log_{\bar{\nu}} \nu_i) \\
    & \quad+ \frac{1}{n} \sum_{i=1}^n\bigl(\PT_{\bar{\nu}_n,\bar{\nu}} (\log_{\bar{\nu}_n} \nu_i)  - \log_{\bar{\nu}} \nu_i\bigr) \otimes\bigl(\PT_{\bar{\nu}_n,\bar{\nu}} (\log_{\bar{\nu}_n} \nu_i)  - \log_{\bar{\nu}} \nu_i\bigr).
\end{align*}
We note that the above proof is to that of \citep[Lemma S1]{chen2023wasserstein}. In the following, we show that \[\norm*{\frac{1}{n} \sum_{i=1}^n\bigl(\PT_{\bar{\nu}_n,\bar{\nu}} (\log_{\bar{\nu}_n} \nu_i)  - \log_{\bar{\nu}} \nu_i\bigr)\otimes\log_{\bar{\nu}} \nu_i}_{\mathcal{H}_{\bar{\nu}}}^2 = o_p(1),\] and the other terms can be shown similarly. The following lemma will be useful.
\begin{lemma} \label{lem:1}
    For any $u_0 \in \Tan_{\mu_0}\PS_2(M)$, $v_1 \in \Tan_{\mu_1}\PS_2(M)$
    \begin{align*}
      \norm{\PT_{0, 1} (u_0) - v_1}_{\mu_1} \leq\norm{\TM_{0, 1} (u_0) - v_1}_{\mu_1}.
    \end{align*}
\end{lemma}
\begin{proof}
For any $0 \leq s < t \leq 1$, $u_s \in \Tan_{\mu_s}\PS_2(M)$, $u_t \in \Tan_{\mu_t}\PS_2(M)$, note that 
\begin{align*}
\norm{\PTM_{s}^t(u_s) - u_t}_{\mu_t} & =\norm*{\Proj_{\mu_{t}}\bigl(\TM_{s, t}(u_s)\bigr) - u_t}_{\mu_t} \\
&\leq\norm{\TM_{s, t} (u_s) - u_t}_{\mu_t}\\
&=\norm{u_s \circ \flow(t, s, \cdot) - u_t}_{\mu_t}\\
&=\norm{u_s - u_t \circ \flow(s, t, \cdot)}_{\mu_s}.
\end{align*}
For any partition $S\in \mathcal{S}$, we apply the above inequality iteratively and obtain
\begin{align*}
&\norm{\PTM_{s_{m-1}}^1\circ\PTM_{s_{m-2}}^{s_{m-1}}\circ\dotsb\circ\PTM_{0}^{s_1}(u_0) - v_1}_{\mu_1} \\
&\quad\leq\norm{\PTM_{s_{m-2}}^{s_{m-1}}\circ\dotsb\circ\PTM_{0}^{s_1}(u_0) - v_1 \circ\flow(s_{m-1}, 1, \cdot)}_{\mu_{s_{m-1}}} \\
&\quad\leq\norm{u_0 - v_1 \circ \flow(s_{m-1}, 1, \cdot) \circ \flow(s_{m-2}, s_{m-1}, \cdot) \circ \cdots \circ \flow(0, s_1)}_{\mu_{s_{0}}} \\
&\quad=\norm{u_0\circ\flow(s_{1}, 0, \cdot)\circ\flow(s_{2}, s_{1}, \cdot)\circ\dotsb\circ\flow(1, s_{m-1}) - v_1 }_{\mu_{1}}\\
&\quad=\norm{u_0 \circ \flow(1, 0, \cdot) - v_1}_{\mu_{1}}.
\end{align*}
Taking limit with respect to all partitions concludes the lemma.
\end{proof}
Using the above lemma, we have
\begin{align*}
&\norm*{\frac{1}{n} \sum_{i=1}^n (\PT_{\bar{\nu}_n,\bar{\nu}} (\log_{\bar{\nu}_n} \nu_i)  - \log_{\bar{\nu}} \nu_i) \otimes  \log_{\bar{\nu}} \nu_i}_{\mathcal{H}_{\bar{\nu}}}^2 \\
&\quad\leq \left\{ \frac{1}{n} \sum_{i=1}^n\norm*{\PT_{\bar{\nu}_n,\bar{\nu}} (\log_{\bar{\nu}_n} \nu_i)  - \log_{\bar{\nu}} \nu_i}_{\bar{\nu}}^2 \right\}\left\{\frac{1}{n} \sum_{i=1}^n \norm{\log_{\bar{\nu}} \nu_i}_{\bar{\nu}}^2 \right\} \\
&\quad\leq \left\{ \frac{1}{n} \sum_{i=1}^n\norm*{\TM_{\bar{\nu}_n,\bar{\nu}} (\log_{\bar{\nu}_n} \nu_i)  - \log_{\bar{\nu}} \nu_i}_{\bar{\nu}}^2 \right\} \left\{ \frac{1}{n} \sum_{i=1}^n\norm{\log_{\bar{\nu}} \nu_i}_{\bar{\nu}}^2 \right\},
\end{align*}
It's easy to see that $n^{-1}\sum_{i=1}^n\norm{\log_{\bar{\nu}} \nu_i}_{\bar{\nu}}^2 =O_p(1) $. For the first part above,
\begin{align*}
    &\frac{1}{n} \sum_{i=1}^n\norm{\cT_{ \bar{\nu}_n, \bar{\nu}} (\log_{\bar{\nu}_n} \nu_i ) -  \log_{\bar{\nu}} \nu_i}_{\bar{\nu}}^2 \\
    &\quad=\frac{1}{n} \sum_{i=1}^n \norm{ (T_{ \bar{\nu}_n \rightarrow \nu_i  } -\id) \circ T_{\bar{\nu} \rightarrow \bar{\nu}_n} -  (T_{\bar{\nu} \rightarrow \nu_i  } -\id)}_{\bar{\nu}}^2 \\
    &\quad= \frac{1}{n} \sum_{i=1}^n \norm{ T_{\bar{\nu}_n \rightarrow \nu_i  }  \circ T_{\bar{\nu} \rightarrow \bar{\nu}_n } - T_{\bar{\nu} \rightarrow \bar{\nu}_n } -  T_{\bar{\nu} \rightarrow \nu_i  } + \id }_{\bar{\nu}}^2 \\
    &\quad\lesssim \frac{1}{n} \sum_{i=1}^n \norm{ T_{\bar{\nu}_n \rightarrow \nu_i  }  \circ T_{\bar{\nu} \rightarrow \bar{\nu}_n} -  T_{\bar{\nu} \rightarrow \nu_i  } }_{\bar{\nu}}^2 + \norm{ T_{\bar{\nu} \rightarrow \bar{\nu}_n }  - \id }_{\bar{\nu}}^2 \\
    &\quad= \frac{1}{n} \sum_{i=1}^n \norm{ T_{\bar{\nu}_n \rightarrow \nu_i  }   -  T_{\bar{\nu} \rightarrow \nu_i  } \circ T_{\bar{\nu}_n \rightarrow \bar{\nu}} }_{\bar{\nu}_n}^2 + \norm{ T_{\bar{\nu} \rightarrow \bar{\nu}_n }  - \id }_{\bar{\nu}}^2 \\
    &\quad= \frac{1}{n} \sum_{i=1}^n \norm{ T_{\bar{\nu}_n \rightarrow \nu_i  }  - T_{\bar{\nu} \rightarrow \nu_i  } + T_{\bar{\nu} \rightarrow \nu_i  } - T_{\bar{\nu} \rightarrow \nu_i  } \circ T_{\bar{\nu}_n \rightarrow \bar{\nu} } }_{\bar{\nu}_n}^2 +\Wass_2^2 (\bar{\nu}_n , \bar{\nu})\\
    &\quad\lesssim \frac{1}{n} \sum_{i=1}^n \norm{ T_{\bar{\nu} \rightarrow \nu_i  }  - T_{\bar{\nu} \rightarrow \nu_i  } \circ T_{\bar{\nu}_n \rightarrow \bar{\nu}} }_{\bar{\nu}_n}^2+ \norm{ T_{\bar{\nu}_n \rightarrow \nu_i  } - T_{\bar{\nu} \rightarrow \nu_i  } }_{\bar{\nu}_n}^2 +\Wass_2^2 (\bar{\nu}_n, \bar{\nu} ) \\
    &\quad= \frac{1}{n} \sum_{i=1}^n \norm{ T_{\bar{\nu} \rightarrow \nu_i  } \circ T_{\bar{\nu} \rightarrow \bar{\nu}_n} - T_{\bar{\nu} \rightarrow \nu_i  } }_{\bar{\nu}}^2 + \norm{ T_{\bar{\nu}_n \rightarrow \nu_i  } - T_{\bar{\nu} \rightarrow \nu_i  } }_{\bar{\nu}_n}^2 +\Wass_2^2 (\bar{\nu}_n, \bar{\nu} ) \\
    &\quad\lesssim \frac{1}{n} \sum_{i=1}^n  \norm{ T_{\bar{\nu}_n \rightarrow \nu_i  } - T_{\bar{\nu} \rightarrow \nu_i  } }_{\bar{\nu}_n}^2 + 2\Wass_2^2 (\bar{\nu}_n, \bar{\nu} ).
\end{align*}
where the last inequality follows from the assumption that  $ T_{\bar{\nu} \rightarrow \nu_i  } $ is $\beta$-Lipschitz for all $i$. For the first term, since the density of $\nu_n$ is bounded above, we have 
\begin{align*}
    &\frac{1}{n} \sum_{i=1}^n  \norm{ T_{\bar{\nu}_n \rightarrow \nu_i  } - T_{\bar{\nu} \rightarrow \nu_i  }}_{\bar{\nu}_n}^2\lesssim \frac{1}{n} \sum_{i=1}^n \int_{M}  \norm{ T_{\bar{\nu}_n \rightarrow \nu_i  } (x) - T_{\bar{\nu} \rightarrow \nu_i} (x)}^2_2\odif{x}.
\end{align*}
In addition, by \citep[Theorem~1.4]{berman2021convergence}, there exits a constant $C>0$ depending on the diameter and volume of $M$, and $c_1$ such that  
$$
\int_{M} \norm{ T_{\bar{\nu}_n \rightarrow \nu_i  } (x) - T_{\bar{\nu} \rightarrow \nu_i} (x)}_{2}^2 \odif{x} \leq C\Wass_2^{2^{1-m}}(\bar{\nu}_n, \bar{\nu} )\quad\forall i.
$$ 
If, in addition,  $\phi_{\bar{\nu} \rightarrow \nu_i  }$ is $\alpha$-strongly convex, \citep[Theorem~1.3]{berman2021convergence} is applied.
\end{proof}

\subsection{Proof of Theorem \ref{thm:2}}
\begin{proof}
For Gaussian distributions $\mu_0\sim N(0, \Sigma_0)$ and $\mu_1\sim N(0, \Sigma_1)$, it holds that \[T_{\mu_0 \rightarrow \mu_1} (x) = A_{\mu_0 \rightarrow \mu_1} x,\] where
\[
A_{\mu_0 \rightarrow \mu_1} = \Sigma_0^{-1/2}\left(\Sigma_0^{1/2}\Sigma_1\Sigma_0^{1/2}\right)^{1/2}\Sigma_0^{-1/2},
\]
which, by \citep[Eqn.~4.12]{bhatia2009positive}, can be equivalently written as
\begin{align*}
  A_{\mu_0 \rightarrow \mu_1} = \Sigma_1^{1/2}\left(\Sigma_1^{-1/2}\Sigma_0^{-1}\Sigma_1^{-1/2}\right)^{1/2}\Sigma_1^{1/2}.
\end{align*}
Since $ T_{\overline{\nu} \rightarrow \nu_i} $ is $C$-Lipshitz, we have from the proof of \Cref{eq:1}, 
\begin{align*}
    &\frac{1}{n} \sum_{i=1}^n \norm[\big]{\mathcal{T}_{ \overline{\nu}_n, \overline{\nu}} (\log_{\overline{\nu}_n} \nu_i ) -  \log_{\overline{\nu}} \nu_i}_{\overline{\nu}}^2\lesssim \frac{1}{n} \sum_{i=1}^n \int\norm{A_{\overline{\nu}_n \rightarrow \nu_i  }x - A_{\overline{\nu} \rightarrow \nu_i  } x}_{2}^2 \odif{\overline{\nu}_n} + 2\Wass_2^2 (\overline{\nu}_n, \overline{\nu} )
\end{align*}
Using $ \norm{\Sigma_i^{1/2}}_2 \leq \sqrt{C} $ and $\norm{\Sigma_i^{-1/2}}_2 \leq 1/\sqrt{c}$, we can get 
\begin{align*}
    &\frac{1}{n} \sum_{i=1}^n \int\norm{A_{\overline{\nu}_n \rightarrow \nu_i  }x - A_{\overline{\nu} \rightarrow \nu_i  } x}_{2}^2\odif{\overline{\nu}}_n \\
    &\quad\leq\frac{1}{n}\sum_{i=1}^n \norm{A_{\overline{\nu}_n \rightarrow \nu_i  }- A_{\overline{\nu} \rightarrow \nu_i  }}_{2}^2\int\norm{x}_2^2\odif{\overline{\nu}}_n \\
    &\quad= \frac{1}{n} \sum_{i=1}^n \norm*{\Sigma_i^{1/2}\left(\Sigma_i^{-1/2}\Sigma_{\overline{\nu}_n}^{-1}\Sigma_i^{-1/2}\right)^{1/2}\Sigma_i^{1/2} - \Sigma_i^{1/2}\left(\Sigma_i^{-1/2}\Sigma_{\overline{\nu}}^{-1}\Sigma_i^{-1/2}\right)^{1/2}\Sigma_i^{1/2}}_{2}^2\int\norm{x}_2^2 \odif{\overline{\nu}}_n \\
    &\quad\lesssim \frac{1}{n}\sum_{i=1}^n \norm*{\left(\Sigma_i^{-1/2}\Sigma_{\overline{\nu}_n}^{-1}\Sigma_i^{-1/2}\right)^{1/2} - \left(\Sigma_i^{-1/2}\Sigma_{\overline{\nu}}^{-1}\Sigma_i^{-1/2}\right)^{1/2}}_{2}^2.
\end{align*}
By Powers--St\o rmer inequality and properties of matrix norms,
\begin{align*}
    \norm*{\left(\Sigma_i^{-1/2}\Sigma_{\overline{\nu}_n}^{-1}\Sigma_i^{-1/2}\right)^{1/2} - \left(\Sigma_i^{-1/2}\Sigma_{\overline{\nu}}^{-1}\Sigma_i^{-1/2}\right)^{1/2}}_{2}^2
    &\leq \norm[\big]{ \Sigma_i^{-1/2}\Sigma_{\overline{\nu}_n}^{-1}\Sigma_i^{-1/2} - \Sigma_i^{-1/2}\Sigma_{\overline{\nu}}^{-1}\Sigma_i^{-1/2}}_{1}\\
    &\lesssim \norm{\Sigma_{\overline{\nu}_n}^{-1}- \Sigma_{\overline{\nu}}^{-1}}_2.
\end{align*}
Also, \citep[Thm.~1]{bhatia2019bures} and Powers--St\o rmer inequality implies that 
\begin{align*}
    \Wass_2^2 (\overline{\nu}_n, \overline{\nu} ) \leq\norm[\big]{\Sigma_{\overline{\nu}_n}^{1/2}- \Sigma_{\overline{\nu}}^{1/2}}_2^2 \lesssim\norm{\Sigma_{\overline{\nu}_n}- \Sigma_{\overline{\nu}}}_2.
\end{align*}
This concludes the theorem.
\end{proof}

\section{Derivation of \texorpdfstring{\Cref{alg:wt_pca}}{Algorithm 1}}
\subsection{WT-PCA as an eigenvalue problem}
In the discrete setting, the first \(L\) principal components \(\xi_1,\dotsc,\xi_L\) are retrieved by solving
\begin{maxi}[2]
{\scriptstyle\xi_1,\dotsc,\xi_L}{\sum_{\ell=1}^{L}\sum_{i=1}^n\inner{T_{\bar{\mu}\to\mu_i}-\id,\xi_\ell}_{\bar{\mu}}^2}{\label{eqn:orig_prob}}{}
\addConstraint{\inner{\xi_k,\xi_\ell}_{\bar{\mu}}}{=\delta_{k\ell}},
\end{maxi}
where we recall the tangential constraints were dropped due to \Cref{lemma:wt-pca_spectral_notanconstraint}. Let \(x_1,\dotsc,x_m\) be samples from \(\bar{\mu}\) and \(\tau_i(x_j)\coloneqq T_{\bar{\mu}\to\mu_i}(x_j)-x_j\) denote the logarithmic map. The objective function of \eqref{eqn:orig_prob} can be written as
\begin{align*}
&\sum_{\ell=1}^{L}\sum_{i=1}^n\inner{T_{\bar{\mu}\to\mu_i}-\id,\xi}_{\bar{\mu}}^2\\
&\quad=\sum_{\ell=1}^{L}\sum_{i=1}^n\left(\int_M\left[T_{\bar{\mu}\to\mu_i}(x)-x\right]^\top\xi(x)\,d\bar{\mu}(x)\right)^2\\
&\quad=\sum_{\ell=1}^{L}\sum_{i=1}^n\left[\sum_{j=1}^m\tau_i(x_j)^\top\xi(x_j)\right]^2\\
&\quad=\sum_{\ell=1}^{L}\sum_{i=1}^n\left[\sum_{j=1}^m\xi(x_j)^\top\tau_i(x_j)\tau_i(x_j)^\top\xi(x_j)+\sum_{a\neq b}^m\xi(x_k)^\top\tau_i(x_a)\tau_i(x_b)^\top\xi(x_\ell)\right]\\
&\quad=\sum_{\ell=1}^{L}\left[\sum_{j=1}^m\xi(x_j)^\top\left(\sum_{i=1}^n\tau_i(x_j)\tau_i(x_j)^\top\right)\xi(x_j)+\sum_{a\neq b}^m\xi(x_k)^\top\left(\sum_{i=1}^n\tau_i(x_a)\tau_i(x_b)^\top\right)\xi(x_\ell)\right].
\end{align*}
Define \[Q_{a,b}\coloneqq\sum_{i=1}^n\tau_i(x_a)\tau_i(x_b)^\top.\] Then, by letting \[\bm{\xi}\coloneqq\begin{pmatrix}
\xi_1(x_1) & \xi_2(x_1) & \dots & \xi_L(x_1)\\
\xi_1(x_2) & \xi_2(x_2) & \dots & \xi_L(x_2)\\
\vdots & \vdots & \ddots & \vdots\\
\xi_1(x_m) & \xi_2(x_m) & \dots & \xi_L(x_m)
\end{pmatrix}\] and \[\bm{Q}\coloneqq[Q_{a,b}]=\begin{pmatrix}
Q_{1,1} & Q_{1,2} & \dots & Q_{1,m}\\
Q_{2,1} & Q_{2,2} & \dots & Q_{2,m}\\
\vdots & \vdots & \ddots & \vdots\\
Q_{m,1} & Q_{m,2} & \dots & Q_{m,m}
\end{pmatrix},\] one can easily check that
\begin{align*}
\operatorname{trace}(\bm{\xi}^\top\bm{Q}\bm{\xi})&=\sum_{\ell=1}^{L}\sum_{i=1}^n\inner{T_{\bar{\mu}\to\mu_i}-\id,\xi}_{\bar{\mu}}^2.
\end{align*} 
It follows then problem \eqref{eqn:orig_prob}, short of the tangential constraint, is equivalent to the eigenvalue problem
\begin{maxi}[2]
{\scriptstyle\bm{\xi}}{\operatorname{trace}(\bm{\xi}^\top\bm{Q}\bm{\xi})}{\label{eqn:eigen_prob}}{}
\addConstraint{\bm{\xi}^\top\bm{\xi}=I}.
\end{maxi}
Note that \(\bm{Q}=\bm{P}^\top\bm{P}\), with \[\bm{P}\coloneqq\begin{pmatrix}
\tau_{1}(x_1) & \tau_{1}(x_2) & \dots & \tau_{1}(x_m)\\
\tau_{2}(x_1) & \tau_{2}(x_2) & \dots & \tau_{2}(x_m)\\
\vdots & \vdots & \ddots & \vdots\\
\tau_{n}(x_1) & \tau_{n}(x_2) & \dots & \tau_{n}(x_m)\\
\end{pmatrix}.\] Therefore, WT-PCA can equivalently be carried out via the singular value decomposition (SVD) of \(\bm{P}\).

\end{document}